\documentclass[a4paper,UKenglish,cleveref, autoref, thm-restate]{lipics-v2019}
\title{The GaussianSketch for Almost Relative Error Kernel Distance} %TODO Please add

\titlerunning{} %TODO optional, please use if title is longer than one line

\author{Jeff M. Phillips}{School of Computing, University of Utah}{jeffp@cs.utah.edu}{}{}%TODO mandatory, please use full name; only 1 author per \author macro; first two parameters are mandatory, other parameters can be empty. Please provide at least the name of the affiliation and the country. The full address is optional

\author{Wai Ming Tai}{School of Computing, University of Utah}{wmtai@cs.utah.edu}{}{}

\authorrunning{J. Phillips and W.M. Tai} %TODO mandatory. First: Use abbreviated first/middle names. Second (only in severe cases): Use first author plus 'et al.'

\Copyright{Jeff M. Phillips and Wai Ming Tai} %TODO mandatory, please use full first names. LIPIcs license is "CC-BY";  http://creativecommons.org/licenses/by/3.0/

\ccsdesc[100]{Theory of computation~Design and analysis of algorithms} %TODO mandatory: Please choose ACM 2012 classifications from https://dl.acm.org/ccs/ccs_flat.cfm 

\keywords{Kernel Distance, Kernel Density Estimation, Sketching} %TODO mandatory; please add comma-separated list of keywords

\category{RANDOM} %optional, e.g. invited paper

\relatedversion{} %optional, e.g. full version hosted on arXiv, HAL, or other respository/website
\supplement{}%optional, e.g. related research data, source code, ... hosted on a repository like zenodo, figshare, GitHub, ...

\acknowledgements{We thank Rasmus Pagh for early conversations on this topic which helped reignite and motivate this line of thought.  Part of the work was completed while the first author was visiting the Simons Institute for Theory of Computing.}%optional

\nolinenumbers %uncomment to disable line numbering

\hideLIPIcs  %uncomment to remove references to LIPIcs series (logo, DOI, ...), e.g. when preparing a pre-final version to be uploaded to arXiv or another public repository

\EventEditors{Jaros{\l}aw Byrka and Raghu Meka}
\EventNoEds{2}
\EventLongTitle{Approximation, Randomization, and Combinatorial Optimization. Algorithms and Techniques (APPROX/RANDOM 2020)}
\EventShortTitle{\mbox{\scriptsize APPROX/RANDOM 2020}}
\EventAcronym{APPROX/RANDOM}
\EventYear{2020}
\EventDate{August 17--19, 2020}
\EventLocation{Virtual Conference}
\EventLogo{}
\SeriesVolume{176}
\ArticleNo{12}
\usepackage{microtype}%if unwanted, comment out or use option "draft"

\usepackage{amsmath}
\usepackage{amsfonts}
\usepackage{amssymb, amsthm, epsfig}
\usepackage{mathrsfs}
\usepackage{bbold}

\usepackage{graphicx}
\usepackage{verbatim}
\usepackage{color}
\usepackage{epstopdf}

\usepackage{mathtools}
\usepackage{float}
\usepackage{algorithm}
\usepackage[noend]{algorithmic}

\usepackage{xspace}
\usepackage{color}

\usepackage{wrapfig}

\newcommand{\Eu}[1]{\ensuremath{\EuScript{#1}}}
\newcommand{\eps}{\varepsilon}
\newcommand{\R}{\ensuremath{\mathbb{R}}}

\newcommand{\etal}{\emph{et al.}\xspace}

\DeclareMathOperator{\poly}{\textsf{poly}}

\newcommand{\Var}{\ensuremath{\textsf{Var}}}
\newcommand{\prob}[1]{\ensuremath{\textbf{{\sffamily Pr}}\hspace{-.8mm}\left[#1\right]}}
\newcommand{\Tr}{\ensuremath{\textsf{Tr}}}

\newcommand{\TS}{\textsc{RecursiveTensorSketch}\xspace}
\newcommand{\GS}{\textsc{GaussianSketch}\xspace}
\newcommand{\GSHD}{\textsc{GaussianSketchHD}\xspace}

\newcommand{\norm}[1]{\left\lVert #1 \right\rVert}
\newcommand\numberthis{\addtocounter{equation}{1}\tag{\theequation}}
\newcommand{\abs}[1]{\left| #1 \right|}

\newcommand{\inner}[2]{\left\langle #1,#2 \right\rangle}

\newcommand{\DK}{\ensuremath{\texttt{D}^2_K}}
\newcommand{\DKo}{\ensuremath{\texttt{D}_K}}

\newcommand{\E}{\ensuremath{\textsf{E}}}

\newcommand{\HK}{\ensuremath{\Eu{H}_K}}
\newcommand{\vol}{\ensuremath{\mathsf{vol}}}

\begin{document}

\maketitle

\begin{abstract}
	We introduce two versions of a new sketch for approximately embedding the Gaussian kernel into Euclidean inner product space.  These work by truncating infinite expansions of the Gaussian kernel, and carefully invoking the RecursiveTensorSketch [Ahle et al. SODA 2020].  
	After providing concentration and approximation properties of these sketches, we use them to approximate the kernel distance between points sets.  These sketches yield almost $(1+\eps)$-relative error, but with a small additive $\alpha$ term.  In the first variants the dependence on $1/\alpha$ is poly-logarithmic, but has higher degree of polynomial dependence on the original dimension $d$.  In the second variant, the dependence on $1/\alpha$ is still poly-logarithmic, but the dependence on $d$ is linear.  
\end{abstract}

\section{Introduction}
Kernel methods are a pillar of machine learning and general data analysis.  These approaches consider classic problems such as PCA, linear regression, linear classification, $k$-means clustering which at their heart fit a linear subspace to a complex data set.  Each of the methods can be solved by only inspecting the data via a dot product $\inner{x}{p}$.  Kernel methods, and specifically the ``kernel trick,'' simply replaces these Euclidean dot products with a non-linear inner product operation.  The two most common inner products are the polynomial kernel $K_z(x,p) = (\inner{x}{p} + 1)^z$ and the Gaussian kernel $K(x,p) = \exp(-\|x-p\|^2)$.  

The ``magic'' of the kernel method works mainly because of the existence of a reproducing kernel Hilbert space (RKHS) $\HK $ associated with any positive definite (p.d.) kernel~\cite{SGFSL10} $K$.
It is a function space, so for any data point $x \in \R^d$, there is a mapping $\phi : \R^d \to \HK$ so $\phi(x) = K(x,\cdot)$.  Since $\phi(x)$ is a function with domain $\R^d$, and each ``coordinate'' of $\phi(x)$ is associated with another point $p \in \R^d$, there are an infinite number of ``coordinates,'' and $\HK $ can be infinite dimensional.  
However, since $\inner{\phi(x)}{\phi(p)}_{\HK} = K(x,p)$, this embedding does not ever need to be computed, we can simply evaluate $K(x,p)$.  
And life was good.  

However, at the dawn of the age of big data, it became necessary to try to explicitly, but approximately, compute this map $\phi$.  Kernel methods typically start by computing and then analyzing the $n \times n$ gram matrix $K_X$ where $(K_X)_{i,j} = K(x_i,x_j)$ for a data sets $X$ of size $n$.  As $n$ became huge, this became untenable.  In a hallmark paper, Rahimi and Recht~\cite{rahimi2007random} devised random Fourier features (RFFs) for p.d. kernels (with max value $1$, e.g., Gaussians) that compute a random map $\tilde \phi : \R^d \to \R^{\tilde D}$ so 
$\inner{\tilde \phi(x)}{\tilde \phi(p)}$ 
is an unbiased estimate of $K(x,p)$, and with probability at least $1-\delta$ has error $|K(x,p) - \inner{\tilde \phi(x)}{\tilde \phi(p)}| \leq \eps$.  For 
just one pair of points they require $\tilde D = O((1/\eps^2) \log (1/\delta))$, or 
for all comparisons among $n$ points $\tilde D_n = ((1/\eps^2) \log (n/\delta))$, or 
for any points in a region $\Lambda$ of volume $\vol(\Lambda)\leq V$, then $\tilde D_V = ((1/\eps^2) \log (V/\delta))$.

However, relative-error-preserving RKHS embeddings for p.d. kernels are impossible without some restriction on the size $n$ or domain $\Lambda$ of the data.  Consider $n$ data points each far from each other so any pair $x,p \in \R^d$ satisfies $K(x,p) < 1/n$.  In any relative-error-approximate embedding $\hat \phi : \R^d \to \R^{\hat D}$, each point must be virtually orthogonal to all other points, and hence $\Omega(n)$ dimensions are required~\cite{LN17}.  
%\jeff{we need a specific result here -- or something more specific about a packing.} 

%Relative-error-preserving RKHS embeddings for p.d. kernels have eluded the community.  Indeed recently, Avron \etal~\cite{AKMMVZ17} shows a lower bound which almost implies that RFFs cannot achieve relative error for inner products.  It shows a $\lambda$-regularized problem requires $\Omega(n/\lambda)$ dimensional RFFs to achieve some relative error, however, requires $\lambda \geq 10/n > 0$; so it stops just short of showing a lower bound for the unregularized version.  
Instead, to obtain (almost) relative-error results in big data sets, researchers have relied on other approaches such as sampling~\cite{WGM17}, exploiting structure of p.d. Gram matrices~\cite{MW17}, devising modified RFFs for regularized kernel regression~\cite{AKMMVZ17}, or building data structures for kernel density estimate queries~\cite{CS17}.  
%\jeff{review, check others, Musco-Musco17? Bach17?}

%%%%%%%%%%%%%%%%%%%%%%%%%%%%%%%%%%%%%%%%%%%%%%%%%%%%%%%%%%%%%%%%%%%%
\subparagraph*{The kernel distance and data set embeddings.}
To address these difficulties, we first turn our attention from the inner product 
$\inner{\phi(x)}{\phi(p)}_{\HK} = K(x,p)$ 
in the RKHS to the natural distance it implies.  Before stating this distance, we generalize the inner product to point sets $P \subset \R^d$ (which extends naturally to probability distributions $\mu_P$ with domain $\R^d$).  We treat $P$ as a discrete probability distribution with uniform $1/|P|$ weight on each point.  This can be represented in $\HK$ as $\Phi(P) = \frac{1}{|P|} \sum_{x \in P} \phi(x)$, known as the kernel mean~\cite{MFSS17}.  Indeed, for any query point $p \in \R^d$, the inner product $\inner{\Phi(P)}{\phi(p)}_{\HK} = \frac{1}{|P|} \sum_{x \in P} K(x,p)$ is precisely the kernel density estimate at $p$.  For two point sets $P,Q \subset \R^d$ we define $\kappa(P,Q) = \frac{1}{|P|}\frac{1}{|Q|} \sum_{x \in P} \sum_{y \in Q} K(p,q) = \langle \Phi(P), \Phi(Q) \rangle_{\HK}$.  

Now the kernel distance~\cite{GI,current} (alternatively known as the current distance~\cite{GlaunesJoshi:MFCA:06} or maximum mean discrepancy~\cite{GBRSS12,smola}) is defined
\[
\DKo(P,Q) = \| \Phi(P) - \Phi(Q)\|_{\HK} 
= \sqrt{\kappa(P,P) + \kappa(Q,Q) - 2 \kappa(P,Q)}.   
\]
Under a slightly restricted class of kernels (a subset of p.d. kernels), called \emph{characteristic} kernels~\cite{SFL11}, this distance is a metric.  These include the Gaussian kernels which we focus on hereafter.  
This distance looks and largely acts like Euclidean distance; indeed, restricted to any finite-dimensional subspace, it is equivalent to Euclidean distance.

In data analysis and statistics, kernel mean is a compact way to represent a point set distribution.
One can also use kernel distance to compare different point set as opposed to more expensive measure such as Wasserstein distance.
In practice, there are various applications such as hypothesis test and geometric search (see section \ref{sec:implications} for detail discussion) that use kernel distance as a core component.
We suggest the reader refer to \cite{sejdinovic2013equivalence, sriperumbudur2016optimal} for more details on the statistical perspective of kernel distance.
Therefore, making computation of the kernel distance scalable by a kernel distance embedding is of significant importance for those downstream applications.
More generally, one can view oblivious kernel distance embedding as special case of oblivious subspace embedding for RKHS \cite{muandet2016kernel, ahle2019oblivious}, which gives a stronger guarantee than a subspace in the RKHS is preserved within relative error.
However, many application of kernel distance do not require such a strong guarantee, which generally attain worse results (see below for more detail comparison).

So a natural question to ask is if this distance is preserved within relative error via some approximate lifting.  Clearly RFFs guarantee additive $\eps$-error.  However, relate this problem to the Johnson-Lindenstrauss (JL) Lemma~\cite{JL84}: JL describe a family of random projections from a high-dimensional space to a $D'$-dimensional space which preserve $(1+\eps)$-relative error on Euclidean distance, again with $D' = O((1/\eps^2) \log(n/\delta))$ for any ${n \choose 2}$ pairs of distances, succeed with probability $1-\delta$, but only guarantees additive error on inner products.  

Moreover, it is possible to apply the JL Lemma to create such an approximate embedding.  First for any set of $n$ points $X$, we can create $n \times n$ Gram matrix $K_X$ (that is positive definite), and decompose it to $K_X = B_X B_X^T$.  Then each row $(B_X)_i$ in $B_X$ is the $n$-dimensional vector representation of the $i$th data point, and the Euclidean distance $\|(B_X)_i - (B_X)_j\|_2$ is the kernel distance between data points $i$ and $j$~\cite{Mer09,Aronszajn1950}.  Then we can apply JL on these rows $\{(B_X)_i\}$ to obtain such an approximate embedding.  However, this embedding is not oblivious to the data (necessary for many big data settings like streaming) and still requires $\Omega(n^2)$ time just to create the Gram matrix, not to mention the time for decomposition.  

Another recent approach~\cite{CP17} analyzed RFFs for this task, and shows that these approximate embeddings do guarantee relative error on the kernel distance, but only between each pair of points $x,p \in \R^d$ (e.g., so $\frac{\|\hat \phi(x) - \hat \phi(p)\|}{\DKo(x,p)} \in (1 \pm \eps)$), and as we describe next many downstream analysis tasks require the distance preserved between point \emph{sets}.  
Alternatively, if we assume $\DK(P,Q) > \alpha$, then standard RFFs can provide a relative error guarantee using $\tilde D = O(\frac{1}{\eps^2 \alpha^2} \log \frac{1}{\delta})$.  However, such a large factor in $\alpha$ is undesirable, since typically $\alpha \ll \eps$.

%%%%%%%%%%%%%%%%%%%%%%%%%%%%%%%%%%%%%%%%%%%%%%%%%%%%%%%%%%%%%%%%%%%
%\subparagraph*{Data Set Embeddings.}

%	A powerful aspect of RKHS embeddings is how they represent point sets $P \subset \R^d$, or in general any probability distribution $\mu_P$ with domain $\R^d$; we hence forth focus on point sets for simplicity.  In particular, if we treat $P$ as a discrete probability distribution with uniform $1/|P|$ weight on each point, then we can represent this in $\HK$ as $\Phi(P) = \frac{1}{|P|} \sum_{x \in P} \phi(x)$, known as the kernel mean~\cite{MFSS17}.  Indeed, for any query point $p$, the inner product $\inner{\Phi(P)}{\phi(p)}_{\HK} = \frac{1}{|P|} \sum_{x \in P} K(x,p)$ is precisely the kernel density estimate at $p$.  Moreover, to run kernel $k$-means we are computing objects which lie in the convex  combination of the points $\{\phi(x_1), \phi(x_2), \ldots\} \subset \HK$.  Hence the Chen \etal~\cite{CP17} result which gives the guarantee mentioned in last paragraph does not guarantee anything about these analysis.  

%\waiming{"The kernel distance has natural extension to these point set representation as" - which "these"?}
%	Using this notation, the kernel distance has a natural extension to these point set representations as 
%	\[
%	\DKo(P,Q) = \| \Phi(P) - \Phi(Q)\|_{\HK} 
%	= \sqrt{\kappa(P,P) + \kappa(Q,Q) - 2 \kappa(P,Q)},
%	\]
%	where $\kappa(P,Q) = \frac{1}{|P|}\frac{1}{|Q|} \sum_{x \in P} \sum_{y \in Q} K(p,q)$.  

%%%%%%%%%%%%%%%%%%%%%%%%%%%%%%%%%%%%%%%%%%%%%%%%%%%%%%%%%%%%%%%%%%%
\subparagraph*{Our Results.}
We provide two sketches $G : \R^d \to \R^D$ for the Gaussian kernel, improving on work of Rahimi and Recht~\cite{rahimi2007random} and Avron \etal~\cite{AKMMVZ17}, which achieves almost relative error for kernel distance.  
Let $F(X) = \frac{1}{|X|}\sum_{x \in X} G(x)$ extend the sketch to point sets $X  \subset \R^d$.  
Then we show that for two point sets $P,Q \subset \R^d$ 
\[
\left|\DK(P,Q) - \|F(P) - F(Q)\|^2 \right| \leq \eps \DK(P,Q) + \alpha.  
\]
As we can always reduce the dimension $G : \R^d \to \R^D$ using JL to about $D = 1/\eps^2$, 
we focus on reducing the runtime dependence, in particular the dependence on $\alpha$.  

In the first sketch (the \GS) to process a single point with $G(x)$ it takes 
$O\left( \frac{d^2}{\eps^2}\log\frac{d}{\eps} +ds\right)$ time, with $s = \Theta\left(\frac{\log ( d \exp(dL^2) / \alpha)}{\log (\frac{1}{L^2} \log(d \exp(dL^2) / \alpha))}\right)$, where $L$ describes the ($L_\infty$) radius of the domain containing $X$.  
%For the kernel distance problem $\tilde O$ hides $\log \frac{1}{\eps}$ factors, and for kernel PCA it also hides $\log n$ factors.  
So the dependence on $1/\alpha$ is less than a single logarithmic term.  

The second sketch (the \GSHD) is useful when the dimension $d$ is potentially large (it turns out to be very similar to a recent sketch in \cite{ahle2019oblivious}, but our analysis is different). 
Then the runtime to compute $G(x)$ is 
$O\left( \frac{s^3}{\eps^2}\log\frac{s}{\eps} + s^2d \right)$ 
where 
$s = \Theta\left(\frac{\log(4 \exp(2 R^2) / \alpha)}{\log (\frac{1}{R^2} \log (4 \exp(2 R^2) / \alpha))}\right)$, and $R$ is the ($L_2$) domain radius.   
Now the dependence on $1/\alpha$ is still poly-logarithmic, but the dependence on dimension $d$ is linear.  

For example, we can set $\alpha = n^{-C_1}$, $R= C_2\sqrt{\log n}$ and $L = C_3\sqrt{\log n}$ for some absolute constant $C_1,C_2,C_3$.
In low dimension, we have $s=\Theta(\frac{\log n }{ \log d})$ and the running time is $O(\frac{d^2}{\eps^2}\log\frac{d}{\eps} +\frac{d\log n}{\log d})$.
In high dimension, we have $s=\Theta(\log n)$ and the running time is $O\left( \frac{1}{\eps^2}\log^3n\log(\log n /\eps) + d\log^2n \right)$.

%%%%%%%%%%%%%%%%%%%%%%%%%%%%%%%%%%%%%%%%%%%%%%%%%%%%%%%%%%%%%%%%%%%
\subparagraph*{Implications.}
Several concrete applications work directly on this kernel distance between point sets. 
First, the kernel two-sample test~\cite{GBRSS12,MFSS17} is a non-parametric way to perform hypothesis tests between two empirical distributions; simply, the null hypothesis of them being drawn from the same distribution is rejected if the kernel distance is sufficiently large.  
While the sketched kernel two-sample test has proven effective under additive error~\cite{FastMMD}, when the significance threshold is $\Theta(1/n)$, the RFF-based solutions require time $O(n^2)$, no better than brute force; but setting $\eps$ constant and $\alpha = 1/n$, our sketches provide near-linear or almost-linear time runtimes.  
Second, devising a Locality Sensitive Hash (LSH) between point sets (or geometrically-aware LSH for probability distributions) has lacked a great general solution.  Despite progress in special cases (e.g., for polygons~\cite{CCL17}, curves~\cite{DS17}), more general distances between geometric distributions, like Earth-Mover distance require $\Omega(\log s)$ distortion on a domain with at least $s$ discrete points~\cite{AIK08}. 
In general, an LSH requires relative error to properly provide $(1+\eps)$-approximate nearest neighbor results.    
In Section \ref{sec:implications} we specify how our new almost relative-error embeddings for the kernel distance provide efficient solutions for these applications.

Furthermore, this embedding can be composed with a Johnson-Lindenstrauss-type embedding~\cite{JL84,AL09,AL11,Ach03,Woo14} to create an overall oblivious embedding of dimension roughly $O(\frac{1}{\eps^2} \log \frac{1}{\delta})$, that is with no dependence on $1/\alpha$ or $d$ (or $n$ or domain radius $L$ or $R$ in the \emph{for each} setting), and roughly the same guarantees.  

\subsection{Comparison to Other Recent Work on Large Data and Kernels}
Recent related works on kernel approximation do not provide our guarantees; we survey here work that addresses similar problems, and often require similar sets of error parameters.

%%%%%%%%%%%%%%%%%%%%%%%%%%%%%%%%%%%%%%%%%%%%%%%%%%%%%%%%%%%%%%%%%%%%
\subparagraph*{Approximated KDEs.}	
Charikar and Siminelakis~\cite{CS17} describe a data structure of size $n \hat D$ and query time $\hat D$, which answers $\kappa(P,t)$ queries within $(1+\eps)$-relative error as long as $\kappa(P,t) > \alpha$; it requires $\hat D = O(\frac{1}{\eps^2} \frac{1}{\sqrt{\alpha}} \log \frac{1}{\delta} e^{O(\log^{2/3} n \log \log n)})$.  However, this cannot argue much about how large $\DKo(P,Q)$ has to be for this to achieve relative error on the kernel distance since it could be $\DKo(P,Q)$ is small but $\kappa(P,t)$ and $\kappa(P,P)$ are both large.  Moreover, its guarantees only work for a single point set $P$ with point queries $t$, not for two or more points sets $P,Q$, as we argue many downstream data analysis tasks require.  

%%%%%%%%%%%%%%%%%%%%%%%%%%%%%%%%%%%%%%%%%%%%%%%%%%%%%%%%%%%%%%%%%%%%
\subparagraph*{Approximated kernel regression.}	
Avron \etal~\cite{AKMMVZ17} modify the RFF embeddings using different sampling probability related to the statistical leverage in the kernel space.  
%It is uniform over a region $[-\tilde \gamma,\tilde \gamma]^d$, with $\tilde \gamma = 10 \sqrt{\log (n/\tilde \lambda)}$ (somewhat similar to a region $[-\gamma,\gamma]^d$ we will define), but then decays smoothly as it extends beyond the area.  
This approximates a $\lambda$-regularized kernel regression problem, creating a $\tilde D$-dimensional embedding; that is for an $n \times n$ gram matrix $K_X$, and a regularization parameter $\lambda$ it creates a $n \times \tilde D$ matrix $Z$ so 
$
(1-\eps) (K_X + \lambda I_n) \preceq Z Z^* + \lambda I_n \preceq (1+\eps) (K_X + \lambda I_n), 
$	
using 	
$
\tilde D 
= 
O(\frac{1}{\eps^2} (L^d \log^{d/2} (n/ \lambda)  + \log^{2d} (n/ \lambda)) \log (s_{ \lambda}(K)/\delta))
$.  
%with typically $s_{\lambda} \approx \sqrt{n}$~\cite{Bach13}.  
Following our forthcoming methods for analysis, one can modify this result to $(1+\eps)$-approximate the kernel distance, with an additive $\alpha$ term, with an embedding of dimension $D = O\left(\frac{1}{\eps^2} (L^d \log^{d/2} \frac{n}{\alpha}  + \log^{2d} \frac{n}{\alpha}) \log \frac{n}{\delta}\right)$.

Also, Ahle \etal \cite{ahle2019oblivious} recently showed that one can create such $\tilde{D}$-dimensional embedding where $\tilde{D} = O(\frac{1}{\eps^2}(R^2 + \log\frac{n}{\eps\lambda})^5s_\lambda(K_X))$ in $O(\frac{1}{\eps^2}(R^2 + \log\frac{n}{\eps\lambda})^6s_\lambda(K_X))$ time for each data point.
Again, in our setting, one can interpret this result as $(1+\eps)$-approximate the kernel distance, with an additive $\alpha$ term, in $O(\frac{1}{\eps^2}(R^2 + \log\frac{n}{\eps\alpha})^6s_\alpha(K_X))$ time.

Compared to our bounds (adapted to our problem using our techniques), these depend on $n$ and $s_\lambda$ (ours do not), the low-d one is exponential in $d$ (ours is polynomial), and the other powers are larger.  
\subparagraph*{Approximate Kernel PCA.}

%and for a point set $X \subset \R^d$ with kernel Gram matrix $K_X = B_X B_K^T$ then $\|B_X - V V^T B_X\|_\zeta^2 \leq (1+\eps) \|B_X - [B]_k\|_\zeta^2 + \alpha$.

Suppose we are given a data set $X = \{x_1, \ldots, x_n\} \subset \R^d$, and want to find a low rank (rank $k$) approximation of $X_\phi = \{\phi(x_1), \phi(x_2), \ldots, \phi(x_n)\} \in \HK$.  
In particular, this can be described concretely in the context of the Gram matrix $K_X$ and its decomposition $B_X B_X^T$.  
Given any $n \times m$ matrix $M$, let $[M]_k$ be its best rank-$k$ approximation.  
A natural question is to find a rank-$k$ matrix $\tilde K_X$ so
\[
\norm{K_X - \tilde{K}_X}_F^2 \leq (1+\eps)\norm{K_X - [K_X]_k}_F^2.  
\]
%\jeff{Explain the Kernel PCA problem here.}	
While most previous work~\cite{drineas2005nystrom,lopez2014randomized,GPP16,SS17,UMMA18} has focused on providing absolute (or additive) error bounds.  For instance, they showed roughly $\|K_X - \tilde{K}_X\|_F^2 \leq \norm{K_X - [K_X]_k}_F^2 + \eps n$ using e.g., Nystr\"om sampling and RFFs.  
%, and roughly $1/\eps^2$ samples of $X$ (i.e., Nystr\"om approaches)~\cite{drineas2005nystrom}
%and RFFs ~\cite{lopez2014randomized} or RFFs and other sketches~\cite{GPP16}.  
%Improved understanding of the statistical rates of these processes has been recently refined~\cite{SS17,UMMA18}, but the error term still grows polynomially with $n$.    
More recently, Musco and Woodruff~\cite{MW17b} for p.d. Gram matrices $K_X$ show how to efficiently find $\tilde{K}_X$ with relative error.  
This only requires $O(nk^{\omega-1} \cdot \poly(\log n/\eps))$ inspections of entries of $K_X$, where $\omega < 2.373$ is the matrix multiplication exponent.  This is not data oblivious, and uses properties of the p.d. matrix, so it does not provide an embedding sketch.

%, a $n \times m$ matrix; that is for any $n \times k$ orthogonal basis matrix $U$, let $\|M - [M]_k\|_\zeta^2 \geq \|M - U U^T M\|_\zeta^2$ for $\zeta \in \{2,F\}$.  
A closely related problem is approximate kernel PCA problem which is to find a $n \times k$ orthonormal matrix $V$ so that
\[
\|B_X - V V^T B_X\|_F^2 \leq (1+\eps)\|B_X - [B_X]_k\|_F^2.  
\] 
%Or nearly equivalently \jeff{WaiMing: can you comment on if these are equivalent or nearly so, the latter is what we prove} that $\|B_X - V V^T B_X\|_F^2 \leq \|B_X - [B_X]_k\|_F^2 + \Eu{E}'$.  
The RKHS basis $V$, provides a compact and non-linear set of attributes to describe a complex data set $X$, and has many uses in analyzing complex data which lacks strong linear correlations.  
Musco and Woodruff~\cite{MW17} provide an algorithm with runtime $O(\mathsf{nnz}(X)) + \tilde O(n^{\omega+1.5} (\frac{k}{\sigma_{k+1} \eps^2})^{\omega - 1.5})$; which has polynomial dependence on $1/\sigma_{k+1}$.  They leave open whether this can be removed or reduced while maintaining only roughly $\mathsf{nnz}(X)$ dependence on $X$.  
The matrix $V$ returned by their algorithm can be used to approximate the matrix $K_X$ by writing $B_XPB_X^T$ where $P$ is the projection onto the row span of $VV^TB_X$.

Our techniques can be combined with the a sketch for the polynomial kernel~\cite{ANW14} to explicitly solve for $V$ so 
\[
\|B_X - V V^T B_X\|_F^2 \leq (1+\eps)\|B_X - [B_X]_k\|_F^2 + \alpha.  
\]
with similar dimensions required for approximating the kernel distance; the $s$ parameter increases roughly by $\log n / \log \log n$.  This is detailed in Appendix \ref{sec:kpca}.  If the data size $n$ has a known bound, then this provides an oblivious sketch for this almost relative error kernel PCA problem.  Moreover, replacing the $\sigma_{k+1}$ with $\eps \alpha$, it almost answers the kernel PCA $\mathsf{nnz}(X)$ question of Musco and Woodruff~\cite{MW17} -- however our algorithm does not depend on the number-of-non-zeros of $X$ through our sketches, so we leave as an open question if our sketches $G(x)$, particular the \GSHD or similar, can be generated in time $O(\mathsf{nnz}(x) \mathsf{polylog}(1/\alpha) + n\mathsf{poly}(k,1/\eps, \log (1/\alpha))$.

%%%%%%%%%%%%%%%%%%%%%%%%%%%%%%%%%%%%%%%%%%%%%%%%%%%%%%%%%%%%%%%%%%%
%%%%%%%%%%%%%%%%%%%%%%%%%%%%%%%%%%%%%%%%%%%%%%%%%%%%%%%%%%%%%%%%%%%
%%%%%%%%%%%%%%%%%%%%%%%%%%%%%%%%%%%%%%%%%%%%%%%%%%%%%%%%%%%%%%%%%%%

\section{The GaussianSketch and its Properties}
In this section we describe our new sketches for approximate mapping from $\R^d$ to an RKHS associated with a Gaussian kernel.  They are based on the \TS of Ahle \etal~\cite{ahle2019oblivious}, so we first review its properties.  

%%%%%%%%%%%%%%%%%%%%%%%%%%%%%%%%%%%%%%%%%%%%%%%%%%%%%%%%%%%%%%%%%%%
\subparagraph*{The RecursiveTensorSketch.}
We first introduce \TS hash family~\cite{ahle2019oblivious}.
Given positive integers $n$, $m$ and $k$, $\TS_{n,m,k}$ is the family of hash functions $T:\mathbb{R}^{n^k}\rightarrow\mathbb{R}^m$ as constructed in \cite{ahle2019oblivious}.
This hash family will be used to construct our main sketch and has the following guarantee \cite{ahle2019oblivious}: suppose $u,v\in\mathbb{R}^{n^k}$ and picking $m=O(\frac{k}{\eps^2})$, then the expectation $\E(\inner{T(u)}{T(v)}) = \inner{u}{v}$ and the variance $\Var(\inner{T(u)}{T(v)}) \leq \frac{\eps^2}{10}\norm{u}^2\norm{v}^2$.
Moreover, the running time of computing $T(x)$ for any $x\in \mathbb{R}^{n^k}$ is $O(km\log m + kn)$.

\subparagraph*{The GaussianSketch.}
Now, we can define the hash family of the first sketch for the Gaussian kernel \GS. 
Given a vector $x\in\mathbb{R}^d$ and a positive integer $s$, we first define $d$ vectors $y^{(1)}_x\dots,y^{(d)}_x\in \mathbb{R}^s$ such that $i$th coordinate of $y^{(j)}_x$ is $\exp(-x_j^2)\sqrt{\frac{2^{i-1}}{(i-1)!}}x_j^{i-1}$.
Given an integer $m$, define $\GS_{m,s}$ to be the family of hash functions that if $G$ is in it, then $G(x) = T(y^{(1)}_x\otimes\cdots\otimes y^{(d)}_x)$ where $T$ is randomly chosen from $\TS_{s,m,d}$.

Here, $x \otimes y$ is Kronecker product.
Namely, given $x\in \mathbb{R}^p$ and $y\in\mathbb{R}^q$, $x\otimes y$ is a $pq$ dimensional vector indexed by two integers $i,j$ where $i=1,\dots,p$ and $j=1,\dots ,q$ such that $(x\otimes y)_{i,j} = x_i\cdot y_j$.
For notational convenience, we extend Kronecker product when $p$ and $q$ are infinity.
Namely, given $\{x_i\}_{i=1}^\infty$ and $\{y_j\}_{j=1}^\infty$ are infinite sequences, $x\otimes y$ is also an infinite sequence indexed by two positive integers $i,j$ such that $(x\otimes y)_{i,j} = x_i\cdot y_j$.
Also, denote $x^{\otimes k} = x\otimes x^{\otimes k-1}$ and $x^{\otimes 0} = 1$.

The rationale for the \GS comes from the following infinite expansion of the Gaussian kernel.
Define $\bar y^{(j)}_x$ (for $j \in [d]$) as the infinite dimensional analog of $y^{(j)}_x$ with its $i$th coordinate as $\exp(-x_j^2)\sqrt{\frac{2^{i-1}}{(i-1)!}}x_j^{i-1}$.

\begin{lemma}\label{lem:gaussian_inner}
	For $x,p \in \R^d$
	\begin{align*}
	&
	\exp(-\norm{x-p}^2) \\
	& = 
	\sum_{j_1=0}^\infty\cdots\sum_{j_d=0}^\infty\left(\exp(-\norm{x}^2)\left(\prod_{i=1}^{d}\sqrt{\frac{2^{j_i}}{j_i!}}x_{i}^{j_i} \right)\right)\left(\exp(-\norm{p}^2)\left(\prod_{i=1}^{d}\sqrt{\frac{2^{j_i}}{j_i!}}p_{i}^{j_i} \right)\right)  
	\\ & = 
	\inner{\bar y^{(1)}_x\otimes\cdots\otimes \bar y^{(d)}_x}{\bar y^{(1)}_p\otimes\cdots\otimes \bar y^{(d)}_p}.
	\end{align*}
\end{lemma}
\begin{proof}
	\begin{align*}
	&
	\exp(-\norm{x-p}^2)\\
	& =
	\exp(-\norm{x}^2)\exp(-\norm{p}^2)\exp(2\inner{x}{p}) \\
	& =
	\exp(-\norm{x}^2)\exp(-\norm{p}^2)\prod_{i=1}^{d}\exp(2x_ip_i) \\
	& =
	\exp(-\norm{x}^2)\exp(-\norm{p}^2)\prod_{i=1}^{d}\left(\sum_{j=0}^\infty \frac{1}{j!}(2x_ip_i)^j \right) \quad \text{by Taylor expansion of $\exp(\cdot)$}\\
	& =
	\exp(-\norm{x}^2)\exp(-\norm{p}^2)\sum_{j_1=0}^\infty\cdots\sum_{j_d=0}^\infty \left(\prod_{i=1}^{d}\frac{1}{j_i!}(2x_{i}p_{i})^{j_i} \right) \\
	& =
	\sum_{j_1=0}^\infty\cdots\sum_{j_d=0}^\infty\left(\exp(-\norm{x}^2)\left(\prod_{i=1}^{d}\sqrt{\frac{2^{j_i}}{j_i!}}x_{i}^{j_i} \right)\right)\left(\exp(-\norm{p}^2)\left(\prod_{i=1}^{d}\sqrt{\frac{2^{j_i}}{j_i!}}p_{i}^{j_i} \right)\right)
	\\ & = 
	\inner{\bar y^{(1)}_x\otimes\cdots\otimes \bar y^{(d)}_x}{\bar y^{(1)}_p\otimes\cdots\otimes \bar y^{(d)}_p}. \qedhere
	\end{align*}
\end{proof}
Note that the Gaussian sketch takes as input one element of these inner products, but trimmed so that each $\bar y^{(j)}_x$ is trimmed to $y^{(j)}_x$ (without the $\;\bar{}\;$ marker) that only has $s$ terms each.

%%%%%%%%%%%%%%%%%%%%%%%%%%%%%%%%%%%%%%%%%%%%%%%%%%%%%%%%%%%%%%%%%%%
\subparagraph*{The GaussianSketchHD.}
We can also define another hash family of sketches for the Gaussian kernel \GSHD, which works better for high dimension $d$, but will have worse dependence on other error and domain parameters.
For $j=1,\dots,s$, it will use $T_j$ as randomly chosen from $\TS_{d,m_j,j-1}$.
Given a vector $x\in\mathbb{R}^d$,  a positive integer $s$, and $s$ positive integers $m_1,\dots,m_s$, define $\GSHD_{m_1,\dots,m_s,s}$ to be the family of hash functions that if $G$ is in it, then $G(x)\in\mathbb{R}^{m}$ with $(m_{j-1}+1)$th coordinate to $m_j$th coordinate be $\sqrt{\frac{2^{j-1}}{(j-1)!}}\exp(-\norm{x}^2)T_j(x^{\otimes j-1}) = T_j(z_x^{(j)}) \in \R^{m_j}$ where $z_x^{(j)} = \sqrt{\frac{2^{j-1}}{(j-1)!}}\exp(-\norm{x}^2) x^{\otimes j-1} \in \R^{d^{j-1}}$ and $m = \sum_{j=1}^s m_j$.
Denote  $z_x$ the $\frac{d^s-1}{d-1}$ dimensional vector where the first coordinate is $z_x^{(1)}$, the next $d$ coordinates are $z_x^{(2)}$, the next $d^2$ coordinates are $z_x^{(3)}$, and so on.
The \GSHD uses the following, a different infinite expansion of the Gaussian kernel (also explored by Cotter \etal~\cite{cotter2011explicit}).  
\begin{lemma}~\label{lem:gaussian_inner_hd}
	For $x,p\in\mathbb{R}^d$,
	\begin{align*}
	\exp(-\norm{x-p}^2) 
	= 
	\sum_{i=0}^\infty \inner{ \exp(-\norm{x}^2)\sqrt{\frac{2^i}{i!}}x^{\otimes i}}{\exp(-\norm{p}^2)\sqrt{\frac{2^i}{i!}}p^{\otimes i}}
	= 
	\sum_{i=0}^\infty \inner{ z_x^{(i)}}{z_p^{(i)}}
	\end{align*}
\end{lemma}
\begin{proof}	
	\begin{align*}
	&
	\exp(-\norm{x-p}^2) \\
	& =
	\exp(-\norm{x}^2)\exp(-\norm{p}^2)\exp(2\inner{x}{p}) \\
	& =
	\exp(-\norm{x}^2)\exp(-\norm{p}^2)\sum_{i=0}^\infty \frac{1}{j!}\left(2\inner{x}{p}\right)^j & \text{by Taylor expansion of $\exp(\cdot)$} \\
	& =
	\exp(-\norm{x}^2)\exp(-\norm{p}^2)\sum_{i=0}^\infty \frac{2^j}{j!}\inner{x^{\otimes j}}{p^{\otimes j}} \\
	& =
	\sum_{j=0}^\infty \inner{ \exp(-\norm{x}^2)\sqrt{\frac{2^j}{j!}}x^{\otimes j}}{\exp(-\norm{p}^2)\sqrt{\frac{2^i}{j!}}p^{\otimes j}}   &\qedhere
	\end{align*}	
\end{proof}

%%%%%%%%%%%%%%%%%%%%%%%%%%%%%%%%%%%%%%%%%%%%%%%%%%%%%%%%%%%%%%%%%%%
%\subparagraph*{Overview of properties of the GaussianSketch and GaussianSketchHD.}
%
\subsection{Concentration Bounds for GaussianSketch and GaussianSketchHD}
The sketches will inherit the concentration properties of the \TS.  Similar observations were recently observed by Ahle \etal~\cite{ahle2019oblivious}.
%; the proofs follow from expansions and rearrangement of terms, and then Chebyshev or Markov inequalities.  
%
Consider a weighted set of elements $X \subset \R^d$ with weights $\alpha_x$ for $x \in X$, and we use the general concentration bounds for these under the \GS.  

\begin{lemma}[\cite{ahle2019oblivious}]\label{lem:GS-Cheb}
	Let $G$ be a randomly chosen hash function in $\GS_{m,s}$ with $m = O\left(\frac{d}{\eps^2}\right)$.  
	Let $v = \sum_{x \in X} \alpha_x y_x^{(1)}\otimes \dots\otimes y_x^{(d)}$, then 
	%\[
	$\E\left[ \norm{ \sum_{x \in X} \alpha_x G(x)}^2 \right] = \|v\|^2$
	and
	%\;\;\; \text{ and } \;\;\;
	$\Var\left[ \norm{ \sum_{x \in X} \alpha_x G(x)}^2 \right]  \leq \frac{\eps^2}{10} \|v\|^4$
	%\]
	and hence with probability at least $9/10$ we have
	$
	\left| \norm{ \sum_{x \in X} \alpha_x G(x)}^2 - \|v\|^2 \right| \leq \eps \|v\|^2.
	$
\end{lemma}

If $G$ is randomly chosen from $\textsc{GaussianSketchHD}_{m_1,\dots,m_s,s}$, then $G(x) = Sz_x$, where $S$ is a $m \times \frac{d^s-1}{d-1}$ random matrix (recall $m=\sum_{j=1}^s m_j$) so, for the $(m_{i-1}+1)$th row to the $m_i$th row, and the $(\frac{d^{i-1}-1}{d-1} + 1)$th column to the $\frac{d^{i}-1}{d-1}$th column forms a matrix $S_i$ where $T_i(z_x^{(i)}) = S_iz_x^{(i)}$, and the rest of entries are zero.
%Finally, denote $M_X$ is a $n\times \frac{d^s-1}{d-1}$ matrix that $i$-th row as $u_{x^{(i)}}$ for given point set $X = \{x^{(1)},x^{(2)}\dots,x^{(n)}\}\subset \mathbb{R}^d$.
%Again, a version of the next concentration bound appears in \cite{ahle2019oblivious}; it involves rewriting and then truncating an expansion of $\E((A_iS^TSB_j^T)^2)$, then applying the Markov inequality.  

\begin{lemma}[\cite{ahle2019oblivious}]~\label{lem:prob_hd}
	Suppose $A,B$ has $\frac{d^s-1}{d-1}$ columns.
	Denote $A_i$ and $B_i$ be $i$th row of $A$ and $B$ respectively.
	By taking $m_i=O\left(\frac{i}{\eps^2}\right)$, we have
	%\[
	$\prob{\norm{AB^T - AS^TSB^T}_F^2 \leq \eps^2\norm{A}_F^2\norm{B}_F^2 } \geq 1-\delta.$  
	%\]
\end{lemma}

\subsection{Truncation Bounds for GaussianSketch and GaussianSketchHD}
These sketches are effective when it is useful to analyze the effect of sketching a large data set $X$ of size $n$, and we desire to show the cumulative measured across all pairs of elements.  For each sketch we expand these infinite sums, and determine the truncation parameter $s$ so the sum of terms past $s$ have a bounded effect.

%We first consider a weighted set of elements $X \subset \R^d$ with weights $\alpha_x$ for $x \in X$, and we want to show general concentration bounds for these under the \GS; the proof is simple and basically in \cite{ahle2019oblivious}.  
%
%\begin{lemma}[\cite{ahle2019oblivious}]\label{lem:GS-Cheb}
%	Let $G$ be a randomly chosen hash function in $\GS_{m,s}$ with $m = O\left(\frac{d}{\eps^2}\right)$.  
%	Let $v = \sum_{x \in X} \alpha_x y_x^{(1)}\otimes \dots\otimes y_x^{(d)}$, then 
%	%\[
%	$\E\left[ \norm{ \sum_{x \in X} \alpha_x G(x)}^2 \right] = \|v\|^2$
%	and
%	%\;\;\; \text{ and } \;\;\;
%	$\Var\left[ \norm{ \sum_{x \in X} \alpha_x G(x)}^2 \right]  \leq \frac{\eps^2}{10} \|v\|^4$
%	%\]
%	and hence with probability at least $9/10$ we have
%	$
%	\left| \norm{ \sum_{x \in X} \alpha_x G(x)}^2 - \|v\|^2 \right| \leq \eps \|v\|^2.
%	$
%\end{lemma}
%\begin{proof}
%	By Avron \etal~\cite{ANW14}, for a \TS $T$ we have $\E[\langle T(u), T(v)\rangle] = \langle u, v \rangle$.  So the expected value preservation follows by setting $u=v$ and linearity of expectation.  
%	
%	Moreover, they shows $\Var[\langle T(u), T(v)\rangle] \leq \frac{\eps^2}{10} \|u\|^2 \|v\|^2$ for $u,v \in \R^{n^k}$, which with $u=v$ is at most $\frac{\eps^2}{10} \|v\|^4$ in our setting.  
%	
%	Then the concentration bound follows by Chebyshev's inequality.  
%\end{proof}

In our analysis, we will use the following inequality which follows by standard calculus analysis, for any $\eta>0$,
\[
\sum_{j=s}^\infty \frac{\eta^{j}}{j!}
\leq 
\frac{\left(\sup_{y\in[-\eta,\eta]}\exp(y)\right)\eta^s}{s!}
\leq
\frac{\exp(\eta)\eta^s}{s!} \qedhere\numberthis\label{eqn:truncate}
\]

The following expression also arises in our analysis.  

\begin{lemma} \label{lem:tail-GS}
	For $\xi,a,b > 0$, setting 
	$s = \Theta\left(\frac{\log \frac{\xi \cdot a}{\alpha}}{\log \left( \frac{1}{b} \log \frac{\xi \cdot a}{\alpha} \right) }\right)$ 
	then the we have
	$\xi \cdot a \left(\frac{b}{s}\right)^s \leq \alpha$.
\end{lemma}
\begin{proof}
	
	%	By setting $\frac{s}{b} = \Theta\left( \frac{\frac{1}{b} \log \frac{\xi a}{\alpha} }{\log \left(\frac{1}{b} \log \frac{\xi a}{\alpha} \right)} \right)$ and multiplying through by the demoninator on the RHS, we have 
	%\[
	%\frac{1}{b} \log \frac{\xi a}{\alpha}
	%\leq 
	%\frac{s}{b} \log (\frac{1}{b} \log \frac{\xi a}{\alpha})
	%\leq
	%\frac{s}{b} \log \frac{s}{b}.
	%\]
	By setting $\frac{s}{b} = C\frac{\gamma}{\log \gamma}$ for some large constant $C$ where $\gamma = \frac{1}{b} \log \frac{\xi a}{\alpha}$, we have
	\[
	\frac{s}{b} \log \frac{s}{b} 
	= 
	C\frac{\gamma}{\log \gamma} \log \left(C\frac{\gamma}{\log \gamma}\right) 
	= 
	\gamma\cdot C\left( 1+ \frac{\log C}{\log \gamma} - \frac{\log \log \gamma}{\log \gamma}\right) 
	\geq 
	\gamma
	=
	\frac{1}{b} \log \frac{\xi a}{\alpha}.
	\]
	Now, if we rearrange the inequality then $\xi \cdot a \left(\frac{b}{s}\right)^s \leq \alpha$.
\end{proof}

Consider a point set $X = \{x^{(1)}, x^{(2)},\dots,x^{(n)}\}\subset\mathbb{R}^d$, denote $K_X$ as the $n\times n$ matrix with $(K_X)_{i,j} = \exp(-\norm{x^{(i)}-x^{(j)}}^2)$.
First truncate $K_X$ using Lemma \ref{lem:gaussian_inner} to obtain the $n\times n$ matrix $K^{\mathsf{GS}}_{X,s}$ with
\begin{align*}
	&
	(K^{\mathsf{GS}}_{X,s})_{i,j} \\
	& = 
	\sum_{j_1=0}^{s-1}\cdots\sum_{j_d=0}^{s-1}\left(\exp(-\norm{x^{(i)}}^2)\left(\prod_{a=1}^{d}\sqrt{\frac{2^{j_a}}{j_a!}}(x^{(i)}_{a})^{j_a} \right)\right)\\
	& \hspace{2in}\cdot\left(\exp(-\norm{x^{(j)}}^2)\left(\prod_{a=1}^{d}\sqrt{\frac{2^{j_a}}{j_a!}}(p^{(j)}_{a})^{j_a} \right)\right) 
\end{align*}
%Namely, $K^{\mathsf{GS}}_{X,s}$ is the truncated version of $K_X$ based on Lemma \ref{lem:gaussian_inner}.

\begin{lemma}~\label{lem:tail}
	Suppose $X \subset \R^d$ so for all $x^{(i)} \in X$ has $\norm{x^{(i)}}_\infty\leq L$ for some $L>0$.
	Given a vector $w\in\mathbb{R}^n$ with $\left(\sum_{i=1}^n \abs{w_i}\right)^2 \leq \xi$, we have 
	\[
	w^T(K_X - K^{\mathsf{GS}}_{X,s})w
	\leq 
	\left(\sum_{i=1}^n \abs{w_i}\right)^2 d\exp(2dL^2)\left(\frac{2eL^2}{s}\right)^s
	\leq
	\alpha,
	\]
	where the last $\leq \alpha$ inequality follows from setting 
	$s = s_{L,d,\alpha} = \Theta\left(\frac{\log \frac{\xi \cdot d \exp(2dL^2)}{\alpha}}{\log \left( \frac{1}{2 e L^2} \log \frac{\xi \cdot d \exp(2d L^2)}{\alpha} \right) }\right)$.
\end{lemma}

%\begin{proof}
%	See Appendix.  Using Lemma \ref{lem:gaussian_inner} we expand the LHS, using only when some dimension $j_b \geq s$.  Using $\|x^{(i)}\|_\infty \leq L$, this can be bounded by $(\sum_{i=1}^n |w_i|)^2$ times a term independent of $X$.  This latter term simplifies to the stated bound through a few technical inequalities.  
%\end{proof}

\begin{proof}
	From Lemma \ref{lem:gaussian_inner}, we have
	\begin{align*}
		&
		(K_X - K^{\mathsf{GS}}_{X,s})_{i,j}\\
		& =  \hspace{-3mm} 
		\sum_{\substack{j_1,\dots,j_d  \\ \text{one of $j_b\geq s$}}}\left(\exp(-\norm{x^{(i)}}^2)\left(\prod_{a=1}^{d}\sqrt{\frac{2^{j_a}}{j_a!}}(x^{(i)}_{a})^{j_a} \right)\right) \\
		& \hspace{2in}\cdot\left(\exp(-\norm{x^{(j)}}^2)\left(\prod_{a=1}^{d}\sqrt{\frac{2^{j_a}}{j_a!}}(x^{(j)}_{a})^{j_a} \right)\right)
	\end{align*}
	Then we can analyze these in aggregate with respect to a test vector $z$.  The first line uses the fact that a matrix $A$ (for instance with $A = K_X - K^{\textsf{GS}}_{X,s}$) written as \\$\sum_{j} (\sum_{x_i \in X} \psi_j(x_i))(\sum_{x'_i\in X} \psi_j(x'_i))$  can be simplified $w^T A w = \sum_j (\sum_{x_i \in X} w_i \psi_j(x_i))^2$.   
	\begin{align*}
		&
		w^T(K_X - K^{\mathsf{GS}}_{X,s})w \\
		& =
		\sum_{\substack{j_1,\dots,j_d  \\ \text{one of $j_b\geq s$}}}\left(\sum_{i=1}^n w_i\exp(-\norm{x^{(i)}}^2)\left(\prod_{a=1}^{d}\sqrt{\frac{2^{j_a}}{j_a!}}(x^{(i)}_{a})^{j_a} \right)\right)^2 \\
		& \leq
		\sum_{b=1}^d\sum_{\substack{j_1,\dots,j_d  \\ j_b\geq s}}\left(\sum_{i=1}^n w_i\exp(-\norm{x^{(i)}}^2)\left(\prod_{a=1}^{d}\sqrt{\frac{2^{j_a}}{j_a!}}(x^{(i)}_{a})^{j_a} \right)\right)^2 & \text{by union bound} \\
		& \leq
		\sum_{b=1}^d\sum_{\substack{j_1,\dots,j_d  \\ j_b\geq s}}\left(\sum_{i=1}^n \abs{w_i}\left(\prod_{a=1}^{d}\sqrt{\frac{2^{j_a}}{j_a!}}L^{j_a} \right)\right)^2 & \text{assuming $\norm{x^{(i)}}_\infty\leq L$} \\
		& \leq
		\left(\sum_{i=1}^n \abs{w_i}\right)^2 \left(\sum_{b=1}^d\sum_{\substack{j_1,\dots,j_d  \\ j_b\geq s}} \left(\prod_{a=1}^{d}\frac{(2L^2)^{j_a}}{j_a!} \right)\right)
	\end{align*}
	The term $\sum_{b=1}^d\sum_{\substack{j_1,\dots,j_d  \\ j_b\geq s}} \left(\prod_{a=1}^{d}\frac{(2L^2)^{j_a}}{j_a!} \right)$ can be expressed as the follows.
	\begin{align*}
	&
	\sum_{b=1}^d\sum_{\substack{j_1,\dots,j_d  \\ j_b\geq s}} \left(\prod_{a=1}^{d}\frac{(2L^2)^{j_a}}{j_a!} \right) \\
	& =
	\sum_{b=1}^d \left(\sum_{j_1=0}^\infty\frac{(2L^2)^{j_1}}{j_1!} \right)\cdots\left(\sum_{j_b=s}^\infty\frac{(2L^2)^{j_b}}{j_b!} \right)\cdots\left(\sum_{j_d=0}^\infty\frac{(2L^2)^{j_d}}{j_d!} \right) \\
	& =
	\sum_{b=1}^d \left(\prod_{\substack{a = 1 \\ a \neq b}}^{d}\exp(2L^2)\right)\left(\sum_{j_b=s}^\infty\frac{(2L^2)^{j_b}}{j_b!} \right) \\
	& \leq
	\sum_{b=1}^d \left(\exp((d-1)2L^2)\right)\frac{\exp(2L^2)(2L^2)^s}{s!} & \text{by (\ref{eqn:truncate})} \\
	& =
	\frac{d\exp(2dL^2)(2L^2)^s}{s!}\\
	& \leq
	d\exp(2dL^2)\left(\frac{2eL^2}{s}\right)^s &\text{by the fact $s!\geq \left(\frac{s}{e}\right)^s$}
	\end{align*}
	
	Thus, we have
	\begin{align*}
	w^T(K_X - K^{\mathsf{GS}}_{X,s})w 
	& \leq 
	\left(\sum_{i=1}^n \abs{w_i}\right)^2 \left(\sum_{b=1}^d\sum_{\substack{j_1,\dots,j_d  \\ j_b\geq s}} \left(\prod_{a=1}^{d}\frac{(2L^2)^{j_a}}{j_a!} \right)\right) \\
	& \leq 
	\left(\sum_{i=1}^n \abs{w_i}\right)^2d\exp(2dL^2)\left(\frac{2eL^2}{s}\right)^s \\
	& \leq 
	\alpha
	\end{align*}
	where the last inequality follows Lemma \ref{lem:tail-GS} using $\xi = \left(\sum_{i=1}^n \abs{w_i}\right)^2$,  $a = d \exp(2dL^2)$ and $b = 2eL^2$.  	
\end{proof}

Now truncate $K_X$ based on Lemma \ref{lem:gaussian_inner_hd} to obtain $K^{\textsf{HD}}_{X,s}$  with
\[
(K^{\textsf{HD}}_{X,s})_{i,j} = \sum_{a=0}^{s-1} \inner{ \exp(-\norm{x^{(i)}}^2)\sqrt{\frac{2^a}{a!}}(x^{(i)})^{\otimes a}}{ \exp(-\norm{x^{(j)}}^2)\sqrt{\frac{2^a}{a!}}(x^{(j)})^{\otimes a}}
\]

\begin{lemma}\label{lem:tail_hd}
	Define $\Lambda_R^d = \{x \in \R^d \mid \|x\|_2 \leq R\}$.  
	For a point set $X \subset \Lambda_R^d$, and a vector $w\in\mathbb{R}^n$ with $(\sum_{i=1}^n \abs{w_i})^2\leq \xi$, we have 
	\[
	w^T(K_X - K^{\textsf{HD}}_{X,s})w \leq \left(\sum_{i=1}^n \abs{w_i}\right)^2 \exp(2R^2)\left(\frac{2eR^2}{s}\right)^s \leq \alpha
	\]
	where the last $\leq \alpha$ inequality follows from setting $s = s_{R,\alpha} = \Theta\left(\frac{\log \frac{\xi \cdot \exp(2R^2)}{\alpha}}{\log \left( \frac{1}{2 e R^2} \log \frac{\xi \cdot \exp(2R^2)}{\alpha} \right) }\right)$.
\end{lemma}

%\begin{proof}
%	See Appendix.  This proof is conceptually and technically similar to that for Lemma \ref{lem:tail}, but starts with an expansion using Lemma \ref{lem:gaussian_inner_hd}, and now requires $\|x\|_2 \leq R$ to separate a factor $(\sum_{i=1}^n |w_i|)^2$.   
%\end{proof}

\begin{proof} 
	
	From Lemma \ref{lem:gaussian_inner_hd}, we have
	\begin{align*}
	(K_X - K^{\textsf{HD}}_{X,s})_{i,j}
	& =
	\sum_{a=s}^{\infty} \inner{ \exp(-\norm{p^{(i)}}^2)\sqrt{\frac{2^a}{a!}}(p^{(i)})^{\otimes a}}{ \exp(-\norm{p^{(j)}}^2)\sqrt{\frac{2^a}{a!}}(p^{(j)})^{\otimes a}}
	\end{align*}
	Then we can analyze these in aggregate with respect to a test vector $z$.
	The first line uses the fact that a matrix $A$ (for instance with $A = K_X - K^{\textsf{HD}}_{X,s}$) written as \\$\sum_{j} (\sum_{x_i \in X} \psi_j(x_i))(\sum_{x'_i\in X} \psi_j(x'_i))$  can be simplified $w^T A w = \sum_j (\sum_{x_i \in X} w_i \psi_j(x_i))^2$.  
	
	\begin{align*}
	&
	w^T(K_X - K^{\textsf{HD}}_{X,s})w \\
	& =
	\sum_{a=s}^{\infty} \norm{ \sum_{i=1}^n w_i\exp(-\norm{x^{(i)}}^2)\sqrt{\frac{2^a}{a!}}(x^{(i)})^{\otimes a}}^2\\
	& \leq
	\sum_{a=s}^{\infty}  \left(\sum_{i=1}^n \abs{w_i}\norm{\exp(-\norm{x^{(i)}}^2)\sqrt{\frac{2^a}{a!}}(x^{(i)})^{\otimes a}}\right)^2 \\
	& \leq
	\sum_{a=s}^{\infty}  \left(\sum_{i=1}^n \abs{w_i}\sqrt{\frac{2^a}{a!}}R^{ a}\right)^2 & \text{assuming $\norm{x^{(i)}}\leq R$} \\
	& =
	\left(\sum_{i=1}^n \abs{w_i}\right)^2 \left( \sum_{a=s}^\infty \frac{(2R^2)^a}{a!} \right) \\
	& \leq
	\left(\sum_{i=1}^n \abs{w_i}\right)^2 \frac{\exp(2R^2)(2R^2)^s}{s!} &\text{by (\ref{eqn:truncate})} \\
	& \leq
	\left(\sum_{i=1}^n \abs{w_i}\right)^2 \exp(2R^2)\left(\frac{2eR^2}{s}\right)^s &\text{by the fact $s!\geq \left(\frac{s}{e}\right)^s$} \\
	& \leq
	\alpha
	\end{align*}
	where the last inequality follows Lemma \ref{lem:tail-GS} using $\xi = \left(\sum_{i=1}^n \abs{w_i}\right)^2$,  $a = \exp(2R^2)$ and $b = 2eR^2$.  
\end{proof}

\section{Application to the Gaussian Kernel Distance}
\label{sec:DK}
Let $K:\R^d\times\R^d \rightarrow \R$ be Gaussian kernel.
Namely, for any $x,y\in\R^d$, $K(x,y)=\exp(-\norm{x-y}^2)$.
Given two point sets $P,Q\subset \R^d$, one can define a similarity function
$
\kappa(P,Q)=\frac{1}{\abs{P}}\frac{1}{\abs{Q}}\sum_{x\in P}\sum_{y\in Q} K(x,y)
$
and a squared kernel distance
\[
\DK(P,Q)=\kappa(P,P)-2\kappa(P,Q)+\kappa(Q,Q).  
\]

We make the important observation that the above formulation is equivalent to the following form which will be much simpler to fit within our framework:
\[
\DK(P,Q)=\sum_{x \in P\cup Q}\sum_{y\in P\cup Q} \beta_x\beta_y \exp(-\norm{x-y}^2)
\]
where $\beta_x$ is $\frac{1}{\abs{P}}$ if $x\in P$ and $-\frac{1}{\abs{Q}}$ if $x\in Q$.

We now express $\DK(P,Q)$ as the infinite sum using Lemma \ref{lem:gaussian_inner}.  
\begin{align*}
	&
	\DK(P,Q) \\
	& =
	\sum_{x \in P\cup Q}\sum_{y\in P\cup Q} \beta_x\beta_y \exp(-\norm{x-y}^2) \\
	& =
	\sum_{x \in P\cup Q}\sum_{y\in P\cup Q} \beta_x\beta_y \sum_{j_1=0}^\infty\cdots\sum_{j_d=0}^\infty\left(\exp(-\norm{x}^2)\left(\prod_{i=1}^{d}\sqrt{\frac{2^{j_i}}{j_i!}}x_{i}^{j_i} \right)\right) \\
	& \hspace{2.5in}\cdot\left(\exp(-\norm{y}^2)\left(\prod_{i=1}^{d}\sqrt{\frac{2^{j_i}}{j_i!}}y_{i}^{j_i} \right)\right) \\
	& =
	\sum_{j_1=0}^\infty\cdots\sum_{j_d=0}^\infty \left( \sum_{x \in P\cup Q} \beta_x\exp(-\norm{x}^2)\left(\prod_{i=1}^{d}\sqrt{\frac{2^{j_i}}{j_i!}}x_{i}^{j_i} \right) \right)^2	
	\\ & = 
	\norm{\sum_{x \in P \cup Q} \beta_x \bar y_x^{(1)} \otimes\cdots\otimes \bar y^{(d)}_x}^2,
\end{align*}
where each $\bar y_x^{(j)}$ is an infinite dimension vector with $i$th coordinate $\exp(-x_j^2) \sqrt{\frac{2^{i-1}}{(i-1)!}} x_j^{i-1}$.

\begin{theorem}\label{thm:main}
	For any $\eps,R,\alpha>0$, let $G$ be randomly chosen from $\GS_{m,s}$ with $m = O\left(\frac{d}{\eps^2}\right)$ and $s = \Theta\left( \frac{\log \frac{4d\exp(2dL^2)}{\alpha}}{\log \left(\frac{1}{2eL^2}\log \frac{4d\exp(2dL^2)}{\alpha}\right)} \right)$.  
	Let $\Omega^d_L = \{x \in \R^d \mid \|x\|_\infty \leq L\}$.  
	Define a mapping function $F$ from any $X \subset \Omega^d_L$ so $F(X) = \sum_{x \in X} G(x)$, which is a vector in $\R^m$.  
	Then for any $P,Q \subset \Omega_L^d$ with probability at least $9/10$
	\[
	\left| \| F(P) - F(Q)\|^2 - \DK(P,Q) \right| \leq \eps \DK(P,Q) + \alpha.
	\]
	The mapping $G : \R^d \to \R^m$ can be computed in $O\left( \frac{d^2}{\eps^2}\log\frac{d}{\eps} + ds \right)$ time.  
\end{theorem}

\begin{proof}	
	To analyze the \GS, we need to account for error from two sources: from the \TS (using Lemma \ref{lem:GS-Cheb}) and parameter $m$, and from the truncation of the Taylor expansion at $s$ (using Lemma \ref{lem:tail}).  
	In this case we analyze the following infinite expansion
	\[
	\DK(P,Q)
	= 
	\norm{\sum_{x \in P \cup Q} \beta_x \bar y_x^{(1)} \otimes\cdots\otimes \bar y^{(d)}_x}^2,
	\]
	where each $\bar y_x^{(j)}$ is an infinite dimension vector with $i$th coordinate $\exp(-x_j^2) \sqrt{\frac{2^{i-1}}{(i-1)!}} x_j^{i-1}$.  
	
	Let $v = \sum_{x \in P \cup Q} \beta_x \bar y_x^{(1)} \otimes\cdots\otimes \bar y^{(d)}_x$. 
	Then by Lemma \ref{lem:GS-Cheb} by setting $m = O(d/\eps^2)$ we have with probability at least $9/10$ that
	\[
	\left| \norm{\sum_{x \in P \cup Q} \beta_x G(x)}^2 - \|v\|^2 \right| \leq \eps \|v\|^2.  
	\]
	
	Next note that $(\sum_{x \in P \cup Q} |\beta_x|)^2 \leq 4 = \xi$.  
	So by Lemma \ref{lem:tail} the truncation by only $s$ terms can be accounted for as
	\[
	\DK(P,Q) - \|v\|^2 = \beta^T \left(K_{P \cup Q} - K^{\mathsf{GS}}_{P \cup Q,s} \right) \beta \leq 4 d \exp(2d L^2) \left(\frac{2 e L^2}{s}\right)^2 \leq \alpha,
	\]
	where $K_{P \cup Q}$ and  $K^{\mathsf{GS}}_{P \cup Q,s}$ are defined as in Lemma \ref{lem:tail} with $X = P \cup Q$.

	Combining these together we have 
	\[
	(1-\eps)\left(\DK(P,Q) - \alpha\right) \leq (1-\eps)\|v\|^2 \leq | F(P) - F(Q) | \leq (1+\eps) \|v\|^2  \leq (1+\eps) \DK(P,Q).
	\]
	and hence as desired
	\[
	\left| \| F(P) - F(Q)\|^2 - \DK(P,Q) \right| \leq \eps \DK(P,Q) + \alpha. 
	\]
	Recall that the running time of $G$ for mapping a point is 
	\[
	O(dm\log m + ds) = O\left( \frac{d^2}{\eps^2}\log\frac{d}{\eps} + ds \right). \qedhere
	\]
\end{proof}

%%%%%%%%%%%%%%%%%%%%%%%%%%%%%%%%%%%%%%%%%%%%%%%%%%%%%%%%%%%%%%%%%%%
\subparagraph*{Using the Gaussian Sketch HD for high dimensions.}
We first express $\exp(-\norm{x-y}^2)$ as another infinite sum using Lemma~\ref{lem:gaussian_inner_hd}.
Starting with \\$\DK(P,Q) = \sum_{x\in P\cup Q} \sum_{y\in P\cup Q} \beta_x\beta_y\exp\left(-\norm{x-y}^2\right)$ where $\beta_x$ is $\frac{1}{\abs{P}}$ if $x\in P$ and $-\frac{1}{\abs{Q}}$ if $x\in Q$, we have
\begin{align*}
\DK(P,Q)
%& =
%\sum_{x\in P\cup Q} \sum_{y\in P\cup Q} \beta_x\beta_y\exp\left(-\norm{x-y}^2\right) \\
& =
\sum_{x\in P\cup Q} \sum_{y\in P\cup Q} \beta_x\beta_y\inner{ \exp(-\norm{x}^2)\sqrt{\frac{2^i}{i!}}x^{\otimes i}}{\exp(-\norm{y}^2)\sqrt{\frac{2^i}{i!}}y^{\otimes i}}\\
& =
\sum_{i=0}^\infty \norm{ \sum_{x\in P\cup Q} \beta_x\exp(-\norm{x}^2)\sqrt{\frac{2^i}{i!}}x^{\otimes i}}^2.  
\end{align*}
%Recall that $\beta_x$ is $\frac{1}{\abs{P}}$ if $x\in P$ and $-\frac{1}{\abs{Q}}$ if $x\in Q$.

\begin{theorem}\label{thm:main_hd}
	For any $\eps,R,\alpha>0$, let $G$ be randomly chosen from \\$\GSHD_{m_1,\dots,m_s,s}$ with $m_i = O\left(\frac{i}{\eps^2}\right)$ and 
	$s=\Theta\left( \frac{\log \frac{4\exp(2R^2)}{\alpha}}{\log \left(\frac{1}{2eR^2}\log \frac{4\exp(2R^2)}{\alpha}\right)} \right)$.  
	Let $\Lambda^d_R = \{x \in \R^d \mid \|x\|_2 \leq R\}$.  
	Define a mapping function $F$ from any $X \subset \Lambda^d_L$ so $F(X) = \sum_{x \in X} G(x)$, which is a vector in $\R^{m}$ where $m = \sum_{i=1}^s m_i$.  
	Then for any $P,Q \subset \Lambda_R^d$ with probability at least $9/10$
	\[
	\left| \| F(P) - F(Q)\|^2 - \DK(P,Q) \right| \leq \eps \DK(P,Q) + \alpha.
	\]
	The mapping $G : \R^d \to \R^{m}$ can be computed in $O(\frac{s^3}{\eps^2}\log\frac{s}{\eps} + s^2d)$ time.  
\end{theorem}

\begin{proof}
	Suppose $G(x)\in\mathbb{R}^{m}$ with $(m_{i-1}+1)$th coordinate to $m_i$th coordinate be \\$\sqrt{\frac{2^{i-1}}{(i-1)!}}\exp(-\norm{x}^2)T_i(x^{\otimes i-1})$.
	Here, $T_i$ is randomly chosen from \\$\TS_{d,m_i,i-1}$ for $i=1,\dots,s$.
	
	We first need to invoke Lemma \ref{lem:prob_hd} to inherit the appropriate concentration bounds from the \TS.  	We use $t \times \frac{d^s-1}{d-1}$ matrices $A$ and $B$ as just row vectors with $t=1$, and let $A=B$.  In particular, define this single row as $z = \sum_{x \in P \cup Q} \beta_x [z_x^{(1)}, z_x^{(2)}, \ldots,  z_x^{(s)}]$, then the conclusion of Lemma \ref{lem:prob_hd} is that with probability at least $1-\delta$
	\[
	\abs{\norm{z}^2 - \norm{\sum_{x\in P\cup Q}\beta_xG(x)}^2}^2 = \norm{ \norm{z}^2 - z S^T S z^T}_F^2 \leq \eps^2 \norm{z}^4.  
	\]
	
	So by Lemma \ref{lem:tail_hd} the truncation by only $s$ terms can be accounted for as
	\[
	\DK(P,Q) - \|z\|^2 = \beta^T (K_{P \cup Q} - K^{\mathsf{HD}}_{P \cup Q,s}) ) \beta \leq 4 d \exp(2d L^2) (2 e L^2/s)^2 \leq \alpha,
	\]
	where $K_{P \cup Q}$ and  $K^{\mathsf{HD}}_{P \cup Q,s}$ are defined as in Lemma \ref{lem:tail_hd} with $X = P \cup Q$.

	Combining these together we have 
	\[
	(1-\eps)(\DK(P,Q) - \alpha) \leq (1-\eps)\|z\|^2 \leq \norm{F(P) - F(Q)}^2 \leq (1+\eps) \|z\|^2  \leq (1+\eps) \DK(P,Q).
	\]
	and hence as desired
	\[
	\left| \| F(P) - F(Q)\|^2 - \DK(P,Q) \right| \leq \eps \DK(P,Q) + \alpha. 
	\]
	
	Recall that the running time of $G$ for mapping a point is $O(\sum_{i=1}^s im_i\log m_i + id) = O(\sum_{i=1}^s \frac{i^2}{\eps^2}\log \frac{i}{\eps} + id) = O(\frac{s^3}{\eps^2}\log\frac{s}{\eps} + s^2d)$.
\end{proof}

%%%%%%%%%%%%%%%%%%%%%%%%%%%%%%%%%%%%%%%%%%%%%%%%%%%%%%%%%%%%%%%%%%%
%%%%%%%%%%%%%%%%%%%%%%%%%%%%%%%%%%%%%%%%%%%%%%%%%%%%%%%%%%%%%%%%%%%
%%%%%%%%%%%%%%%%%%%%%%%%%%%%%%%%%%%%%%%%%%%%%%%%%%%%%%%%%%%%%%%%%%%
\section{Extensions and Data Analysis Implications}
\label{sec:implications}
There are many data analysis applications where useful sketched bounds almost immediately follow from this new embedding.  Before we begin, we start by improving the dimensionality of the embedding with a simple post-processing.  We can applying a Johnson-Lindenstrauss-type embedding~\cite{JL84,AL09,AL11,Ach03} to the $m$-dimensional space to obtain $O(1/\eps^2)$-dimensional space that, with constant probability, preserves the distance of a pair of point sets. Furthermore, we can use median trick to boost the success probability to $1-\delta$ by running $O(\log \frac{1}{\delta})$ independent copies. For applications in kernel two-sample hypothesis testing and nearest neighbor searching, setting $\delta$ depends on the number of queries $q$ we make, for instances the bounded number needed for $k$-means clustering~\cite{CEMMP15}, now applied to kernel $k$-means.    
%For other applications like kernel $k$-means clustering, we can set $m = \min\{k, \log n\}$~\cite{CEMMP15}.  
%In general this takes time $O(nD\rho)$ to project all $n$ vectors of dimension $D$ to ones of dimension $\rho$, although when $\rho = O(d^{1/2-\eta})$ for some $\eta > 0$, then the runtime can be improved to $O(n D \log \rho)$~\cite{AL09,AL11}.  
These results are useful for reducing the \emph{storage space} of data representations.  
Recall that the running time of JL embedding from $m$-dimensional space to $\rho$-dimensional space is $O(m \log \rho + \rho^2)$ \cite{AL09,AL11}.

%However, they in general do not lead to faster algorithms since often the cost is already dominated by $O(nD)$ time to generate the $D$-dimensional vectors, and the $O(n D \rho)$ or $O(n D \log \rho)$ time to reduce the dimension would only increase that runtime.  
%Hence, because these properties follow fairly directly from previous results, we do not explicitly restate the bounds.  

%In general for (kernel) regression and classification type results, a strong embedding know as an oblivious subspace embedding (OSE)~\cite{clarkson2013low,NN13,Woo14} is usually required.  This requires $\rho = O(D/\eps^2)$ dimensions when starting from $D$-dimensions, so this does not reduce the dimension further, rather it would increase the dimension.   

%%%%%%%%%%%%%%%%%%%%%%%%%%%%%%%%%%%%%%%%%%%%%%%%%%%%%%%%%%%%%%%%%%%%
\subsection{Kernel Two-Sample Test}
%\waiming{For kernel two sample test it is worsth noting that typically, the error in how well the empirical MMD approximates the true MMD is additive. So it seems that in this setting, getting relative error may not be needed. However, I think relative error is still justifiable. If e.g. P = Q, so the true MMD is 0, the empirical MMD will be something like 1/n (There are generic bounds available (see e.g, Gretton et al 2012) but it is instructive to think about the case when P = Q just places probability 1/2 on two points that are very far from each other.) In this setting, to detect the 1/n threshold with additive error methods would require poly(n) samples (so you might as well just compute the full kernel matrix). However, your method can do it with log n dependence (setting alpha = 1/n and epsilon to a constant).}
The kernel two-sample test~\cite{GBRSS12} is a ``non-parametric'' hypothesis test between two probability distributions represented by finite samples $P$ and $Q$; let $n = |P \cup Q|$.  Then this test simply calculates $\DKo(P,Q)$, and if the value is large enough it rejects the null hypothesis that $P$ and $Q$ represent the same distribution.  Since its introduction a few year ago it has seen many applications and relations; see the recent 140 page survey~\cite{MFSS17}.  
Zhao and Deng~\cite{FastMMD} proposed to speed this test up for large sets using RFFs which improves runtime and in some cases even statistical power.  While several improvements are suggested~\cite{ZGB13} including using FastFood~\cite{LSS13}, these all only provide additive $\eps$-error.  

Consider $P \sim \mu_P$ and $Q \sim \mu_Q$.  If $\mu_P = \mu_Q$, then empirical distributions $P,Q$ may have $\DKo(P,Q) = \Theta(1/n)$.  Hence distinguishing the case of $\mu_P = \mu_Q$ from them not being equal would either require additive error $\eps = \Theta(1/n)$, or relative $(1+\eps)$-error with a minimum $\Theta(1/n)$ additive error.  RFFs would require $\Theta(1/\eps^2) = \Theta(n^2)$ dimensions, so one may just as well compute $\DKo(P,Q)$ exactly in $O(n^2)$ time.  
In our approach, we can set $\eps$ to be a constant (say $\eps = 0.2$) and $\alpha$ to be $\Theta(1/n)$.  Assuming a constant region diameter, the total running time is $O\left(\frac{n\log n}{\log \log n}\right)$ in the low dimensional case (by Theorem \ref{thm:main}) or $O\left(\frac{n\log^2n(\log n +d)}{\log^2\log n}\right)$ in the high dimensional case (by Theorem \ref{thm:main_hd}).
%For example, when $L$ and $d$ are fixed (say data drawn from a spatial domain so $d=2$, and the region diameter is 20 times the kernel bandwidth so $L=10$) then it takes only $O(n \log n)$ time.  

%\waiming{check}
%Consider these two situations.
%The first one is $P=Q$.
%The second one is that all points of $P\cup Q$ are far away from each other with $\abs{P}=\abs{Q}=n/2$.
%
%If the distributions are known to inhabit a bounded region $[-L,L]^d$, then we can invoke Theorem \ref{thm:L-bound} to obtain two vectors of size $D_L=O\left(\left(L^d\left(\log \frac{1}{\eps\alpha}\right)^{d/2}+ \left(\log\frac{1}{\alpha}\right)^{d/2}\left(\log \frac{1}{\eps\alpha}\right)^{d/2}\right)\frac{1}{\eps^2}\log \frac{1}{\delta}\right)$ in $O(nD)$ time, and compute a kernel means $\Phi(P)$ and $\Phi(Q)$ for each, and then compare their distances in $O(D)$ time.  Hence the entire estimate takes $O(nD)$ time.  
%As long as $\DKo(P,Q) > \alpha$, this guarantees the estimate of $\DKo(P,Q)$ is within $\eps \DKo(P,Q)$ of the correct value.  

Another way to determine if $\DKo(P,Q)$ should estimate $P$ and $Q$ as distinct, is to run permutation tests.  That is for some large number (e.g., $q=1000$) of trials, select two sets $P_j, Q_j$ iid from $P \cup Q$, of size $|P|$ and $|Q|$ respectively.  For each generated pair we calculate (or estimate using Theorem \ref{thm:main} or Theorem \ref{thm:main_hd}) the value of $\DKo(P_j,Q_j)$, and then use the $95$th-percentile of these values as a threshold.  Note since each $P_j, Q_j$ is drawn from the same domain as $P,Q$, then the guarantees on the accuracy of the featurized estimate carries over directly even under a large $q$ number of permutations.

%%%%%%%%%%%%%%%%%%%%%%%%%%%%%%%%%%%%%%%%%%%%%%%%%%%%%%%%%%%%%%%%%%%%
%\subparagraph*{LSH for point sets, geometric distributions.}
\subsection{LSH for Point Sets, Geometric Distributions}
%\waiming{In the LSH section, I'm unclear on how you compute the kernel distance between two shapes (is it a between uniform distributions over these shapes?). And how would you approximate this using your methods since it typically has infinite points?}

The new results also allow us to immediately design LSH and nearest neighbor structures for the kernel distance by relying on standard Euclidean LSH~\cite{AI06}.  
Building a search engine for low-dimensional shapes~\cite{FMKCHDJ03} has long been a goal in computational geometry and geometric modeling.  A difficulty arises in that many of the best-known shape distance measures require an alignment (e.g., Frechet~\cite{TEHM1994,AG96} or earth movers~\cite{BNNR11}) which creates many challenges in designing LSH-type procedures.  Some methods have been designed, but with limitations, e.g., on point set size for earth mover distance~\cite{AIK08} or number of segments in curves for discrete Frechet~\cite{DS17}.    
The kernel distance provides an alternative distance for shapes, low-dimensional distributions, or curves~\cite{current}; it can encode normals or tangents as well to encode direction information of curves~\cite{GlaunesJoshi:MFCA:06}.  That is, given two shapes composed of (or approximated by) point sets $P_i, P_j$, the distance between the shapes is simply $\DKo(P_i,P_j)$.  

%Now after embedding to $D_n$ dimensions for a bounded number of points (via Theorem \ref{thm:n-bound}) or to $D_L$ dimensions for objects from a bounded $[-L,L]^d$ domain (via Theorem \ref{thm:L-bound}), standard Euclidean LSH can be applied~\cite{AI06}.  
Given a family of point sets $\Eu{P} = \{P_1, P_2, \ldots, P_N\}$ such that each $P_i \subset \R^d$ has size at most $n$, an $\eps$-approximate nearest neighbor of a query point set $Q$ is a point set $\hat P \in \Eu{P}$ so that $\DKo(\hat P,Q) \leq (1+\eps) \min_{P_j \in \Eu{P}} \DKo(P_j,Q)$.  
Here, we assume that $\DKo(P_i,P_j) \geq \alpha'$ for any $i\neq j$.
For $\eps \leq 1/2$, we can embed each $P_j$ to $F(P_j) \in \R^D$, and then invoke the key result from Andoni and Indyk~\cite{AI06} for a $c'$-approximate nearest neighbor, so the total error factor is $c'(1+\eps)$.  
Overall, we can retrieve a $c$-approximate nearest neighbor (setting $c = c'(1+\eps)$) to a query $Q \subset \R^d$ with $O(D N^{1/c^2 + o(1)})$ query time after using $O(D N^{1+1/c^2 + o(1)})$ space and $O(D N^{1+1/c^2 + o(1)} + N(\frac{n\log \frac{1}{\eps\alpha'}}{\log \log \frac{1}{\eps\alpha'}} + \frac{1}{\eps^2}\log \frac{1}{\eps}))$ preprocessing when $d$ is small or $O(D N^{1+1/c^2 + o(1)} + Nn(\frac{\log^2\frac{1}{\eps\alpha'}(\log \frac{1}{\eps\alpha'}+d)}{\eps^2\log^2\log\frac{1}{\eps\alpha'}}) )$ preprocessing when $d$ is large, both assuming a data region with constant diameter.

%%%%%%%%%%%%%%%%%%%%%%%%%%%%%%%%%%%%%%%%%%%%%%%%%%%%%%%%%%%%%%%%%%%%
%%%%%%%%%%%%%%%%%%%%%%%%%%%%%%%%%%%%%%%%%%%%%%%%%%%%%%%%%%%%%%%%%%%%
%\section{Conclusion}
%\label{sec:conclude}
%
%We have devised and analyzed two new sketches for approximately embedding a Gaussian kernel function $K(p,\cdot)$ into an RKHS, where these sketched vectors can simply be averaged to also approximate the kernel density estimate $\frac{1}{|P|} \sum_{p \in P} K(p,\cdot)$ of a point set $P \subset \R^d$.  The two sketches can provide $(1+\eps)$-relative error of the natural metric in the RKHS, with a small additive term $\alpha$.  
%The runtime of generating the sketches are either sub-logarithmic in $1/\alpha$ with higher polynomial term in $d$, or are linear in $d$ and poly-logarithmic in $1/\alpha$.  The time to sketch a single item does not depend on the size of $|P|$, and can always be reduced roughly $1/\eps^2$ dimensions (independent on $d$ and $\alpha$) using JL.  These new embeddings imply new algorithmic results in important kernel-based data analysis including kernel $k$-means, kernel two-sample test, and nearest neighbors for the kernel distance.  

%%
%% Bibliography
%%

%% Please use bibtex, 

%%%%%%%%%%%%%%%%%%%%%%%%%%%%%%%%%%%%%%%%%%%%%%%%%%%%%%%%%%%%%%%%%%%
%%%%%%%%%%%%%%%%%%%%%%%%%%%%%%%%%%%%%%%%%%%%%%%%%%%%%%%%%%%%%%%%%%%
%%%%%%%%%%%%%%%%%%%%%%%%%%%%%%%%%%%%%%%%%%%%%%%%%%%%%%%%%%%%%%%%%%%
%%%%%%%%%%%%%%%%%%%%%%%%%%%%%%%%%%%%%%%%%%%%%%%%%%%%%%%%%%%%%%%%%%%
%\bibliographystyle{plain}
\bibliography{gaussiansketch}

\appendix

\section{Gaussian Kernel PCA}
\label{sec:kpca}
Let $k$ be a positive integer and $\eps>0$.
Avron \etal~\cite{ANW14} provide the following algorithm.
Suppose $S$ and $T$ are randomly chosen from $\TS_{s,m,d}$ and \\$\TS_{s,r,d}$ respectively where $m=\Theta(d(k^2+\frac{k}{\eps}))$ and $r=\Theta(\frac{dm^2}{\eps^2})$.
Given $n$ vectors $v^{(1)},\dots,v^{(n)} \in \mathbb{R}^{s^d}$, compute $n \times m$ matrix $M$ with $i$th row as $S(v^{(i)})$ and $n \times r$ matrix $N$ that $i$th row as $T(v^{(i)})$.
Let $U$ be the orthonormal basis for column space of $M$ and $W$ be $m\times k$ matrix containing top $k$ left singular vector of $U^TN$.
Finally, return $V=UW$.
This algorithm has the following guarantee.

\begin{lemma}[\cite{ANW14} with straightforward modification]~\label{lem:rankk}
	Given a $n$-by-$s^d$ matrix $A$, a positive integer $k$ and $\eps>0$.
	The above algorithm that has rows of $A$ as input returns a matrix $V$ such that
	\[
	\norm{A - VV^TA}_F^2 \leq (1+\eps)\norm{A-[A]_k}_F^2
	\]
	where $[A]_k$ is the best rank-$k$ approximation of $A$.
\end{lemma}

Now, we can directly modify the above algorithm into our context for 
rank-$k$ Gaussian low-rank approximation.
Given a point set $X=\{x_1,\dots,x_n\}\subset \mathbb{R}^d$ and a positive integer $s$. 
Suppose $G$ and $H$ are randomly chosen from $\textsc{GaussianSketch}_{m,s}$ and $\textsc{GaussianSketch}_{r,s}$ respectively.
Recall that $m=\Theta(d(k^2+\frac{k}{\eps}))$ and $r=\Theta(\frac{dm^2}{\eps^2})$.
Compute the $n \times m$ matrix $M$ with $i$th row as $G(x_i)$ and $n \times r$ matrix $N$ with $i$th row as $H(x_i)$.
Let $U$ be the orthonormal basis for column space of $M$ and $W$ be $m\times k$ matrix containing top $k$ left singular vector of $U^TN$.
Finally, return $V=UW$.

\begin{theorem}~\label{thm:main_pca}
	%For any $\eps,L,\alpha>0$, let $G$ and $H$ be randomly chosen from $\GS_{m,s}$ and $\GS_{r,s}$ with 
	%$m = \Theta(3^d(k^2+\frac{k}{\eps}))$, 
	%$r=\Theta(\frac{3^dm^2}{\eps^2})$, and 
	%$s = \Theta\left( \frac{\log \frac{4n^2d\exp(2dL^2)}{\alpha}}{\log \left(\frac{1}{2eL^2}\log \frac{4n^2d\exp(2dL^2)}{\alpha}\right)} \right)$. 
	Let $\eps,L,\alpha>0$ and $s = \Theta\left( \frac{\log \frac{4n^2d\exp(2dL^2)}{\alpha}}{\log \left(\frac{1}{2eL^2}\log \frac{4n^2d\exp(2dL^2)}{\alpha}\right)} \right)$.
	For $\Omega^d_L = \{x \in \R^d \mid \|x\|_\infty \leq L\}$ and $X \subset \Omega_L^d$, and let $A_X$ be a pd matrix with elements $(A_X)_{i,j} = K(x_i,x_j) = \exp(-\norm{x_i-x_j}^2)$ for $x_i,x_j \in X$ and factorization $A_X = B_X B_X^T$.  
	%Compute the $n \times k$ matrix $V = UW$ as the product of orthogonal basis $U,W$ from $n \times m$ matrix $M$ with rows $G(x_i)$ and $n \times r$ matrix $N$ with rows $H(x_i)$ using $x_i \in X$ as described above.  
	Then with constant probability
	\[
	\norm{B_X - VV^T B_X}_F^2 \leq (1 + \eps) \norm{B_X - [B_X]_k}_F^2 + \alpha.
	\]
	The runtime to compute $V$ is $O\left(nds + n\frac{d^4(k^2+\frac{k}{\eps})^3}{\eps^2}\right)$.
\end{theorem}

\begin{proof}
	Let $v_x^{(i)}$ be a vector in $\mathbb{R}^s$ with $j$th coordinate to be $\exp(-x_i^2)\sqrt{\frac{2^{j-1}}{(j-1)!}} x^{j-1}_i$ for any $x\in \mathbb{R}^d$.
	
	By Lemma~\ref{lem:rankk}, taking $A_s$ as an $n \times s^d$ matrix with $i$th row as $v_{x_i}^{(1)}\otimes\cdots\otimes v_{x_i}^{(d)}$.
	We have
	\[
	\norm{A_s - VV^TA_s}_F^2 \leq (1+\eps)\norm{A_s-[A_s]_k}_F^2
	\] 
	
	From Lemma \ref{lem:tail}, $v^T(B_XB_X^T - A_sA_s^T)v \leq \left(\sum_{i=1}^n \abs{v_i}\right)^2 d\exp(2dL^2)\left(\frac{2eL^2}{s}\right)^s \leq \alpha/n$.
	To see this expression is at most $\alpha/n$, first observe that columns of $V$ are orthonormal, and therefore, the norm of each row of $I-VV^T$ is at most $2$.  Hence, $\left(\sum_{i=1}^n \abs{v_i}\right)^2 \leq 4n$.  Then the choice of $s$ and Lemma \ref{lem:tail-GS} with $\xi = 4n^2$, $a= d \exp(2dL^2)$ and $b = 2 e L^2$ complete this derivation.  
	
	We now have
	\begin{align*}
	\norm{B_X-VV^TB_X}_F^2
	& =
	\Tr((I-VV^T)B_XB_X^T(I-VV^T)^T) \\
	& \leq
	\norm{A_s-VV^TA_s}_F^2 + \Tr((I-VV^T)(B_XB_X^T - A_sA_s^T)(I-VV^T)^T)  \\
	& \leq
	\norm{A_s-VV^TA_s}_F^2 + \alpha
	\end{align*}
	
	On the other hand, by Lemma \ref{lem:gaussian_inner}, $B_XB_X^T - A_sA_s^T$ is still positive definite. 
	Therefore, 
	\begin{align*}
	&
	\norm{A_s-[A_s]_k}_F^2 \\
	& =
	\norm{A_s-UU^TA_s}_F^2 \quad \text{where $U$ is the matrix of top-$k$ left singular vectors of $A_s$} \\
	& \leq
	\norm{A_s - U'U'^TA_s}_F^2 \quad \text{where $U$ is the matrix of top-$k$ left singular vectors of $B_X$}\\
	& =
	\norm{B_X-[B_X]_k}_F^2 - \Tr((I-U'U'^T)(B_XB_X^T - A_sA_s^T)(I-U'U'^T)) \\
	& \leq
	\norm{B_X-[B_X]_k}_F^2 \quad \text{recall that $B_XB_X^T - A
		_sA_s^T$ is positive definite}
	\end{align*}
	
	We can plug in everything.
	\begin{align*}
	\norm{B_X-VV^TB_X}_F^2
	& \leq
	\norm{A_s-VV^TA_s}_F^2 + \alpha \\
	& \leq
	\norm{A_s-[A_s]_k}_F^2 + \alpha \\
	& \leq
	\norm{B_X-[B_X]_k}_F^2 + \alpha.  
	\end{align*}	
	
	To see the running time, it takes $O(d(s+m\log m))$ to compute $G(\cdot)$ and $O(d(s+r\log r))$ time to compute $H(\cdot)$, and hence $n$ times as much to compute matrices $M$ and $N$.  We can compute the basis $U$ of $M$ in $O(nm^2)$ time, and the projection $U^T N$ in $O(nrm)$ time.  The basis $W$ takes $O(rm^2)$ time, and the final low rank basis $V = UW$ takes $O(nmk)$ time.  
	Thus the total runtime is 
	$O(nd(s + m \log m + r \log r) + nm^2 + nrm + rm^2 + nmk) = O(nd(s + rm))$ using that $r > m^2 > k^4$ that $m > \log r$, and assuming $n > r$.  Now using $m = O(d(k^2 + k/\eps))$ and $r = O(d m^2/\eps^2) = O(d^3 (k^4 + k^2/\eps^2)/\eps^2)$ and we have a total time of
	$O\left(nds + n\frac{d^4(k^2+\frac{k}{\eps})^3}{\eps^2}\right)$.	
\end{proof}

\subparagraph*{Gaussian Low Rank Approximation with Gaussian Sketch HD in High Dimensions.}
Now, we can also modify the above algorithm into our context for 
rank-$k$ Gaussian low-rank approximation in another way.
Given a point set $X=\{x_1,\dots,x_n\}\subset \mathbb{R}^d$ and a positive integer $s$. 
Suppose $G$ and $H$ are randomly chosen from $\textsc{GaussianSketchHD}_{m_1,\dots,m_s,s}$ and $\textsc{GaussianSketchHD}_{r_1,\dots,r_s,s}$ respectively.
Here, $m_i=\Theta(i(k^2+\frac{k}{\eps}))$ and $r_i=\Theta(\frac{im^2}{\eps^2})$ where $m = \sum_{i=1}^s m_i$.
Compute the $n \times m$ matrix $M$ with $i$th row as $G(x_i)$ and $n \times r$ matrix $N$ with $i$th row as $H(x_i)$.
Let $U$ be the orthonormal basis for column space of $M$ and $W$ be $m\times k$ matrix containing top $k$ left singular vector of $U^TN$.
Finally, return $V=UW$.

Note that a hash function in $\textsc{GaussianSketchHD}$ is not directly applying a hash function in $\TS$.
Therefore, Lemma~\ref{lem:rankk} cannot be directly applied.
However, we can still exploit the structure of it in order to prove the same lemma.

As Avron \etal~\cite{ANW14} suggest, it is generally possible by combining Lemma~\ref{lem:prob_hd} and arguments in~\cite{ANW14,clarkson2009numerical,kannan2014principal}.
We have the following lemma.
Here, denote $A_s$ is a $n\times \frac{d^s-1}{d-1}$ matrix that $i$th row as $z_{x_i}$ for given point set $X = \{x_1,x_2\dots,x_n\}\subset \mathbb{R}^d$.

\begin{lemma}\label{lem:pca_hd}
	Given a point set $X\subset\mathbb{R}^d$, a positive integer $k$ and $\eps>0$.
	The above algorithm returns a matrix $V$ such that
	\[
	\norm{A_s - VV^TA_s}_F^2 \leq (1+\eps)\norm{A_s-[A_s]_k}_F^2
	\]
	where $[A]_k$ is the best rank-$k$ approximation of $A$.
\end{lemma}

Before getting into Lemma \ref{lem:pca_hd}, the following lemma from \cite{ANW14} which is implied by Lemma \ref{lem:prob_hd} would be helpful.

\begin{lemma}[\cite{ANW14} implied by Lemma \ref{lem:prob_hd} with straightforward modification]\label{lem:kspace}
	For any positive integer $k'$, given any $\frac{d^s-1}{d-1} \times k'$ matrix $B$ with orthonormal columns, we have $\norm{B^TS^TSB - I}_2 \leq \eps$.
	Here, $S$ is randomly chosen from $\GSHD_{n_1,\dots,n_s,s}$ where $n_i = \frac{ik'^2}{\eps^2}$.
\end{lemma}

\begin{proof}(of Lemma \ref{lem:pca_hd})
	
	In the proof of Theorem 3.1 from \cite{clarkson2009numerical}, the only properties of $S$ used are
	\begin{itemize}
		\item Given any $\frac{d^s-1}{d-1} \times k$ matrix $B$ with orthonormal columns, we have $\norm{B^TS^TSB - I}_2 \leq \eps_0$ for some constant $\eps_0>0$
		\item For any two matrices $A,B$ with $\frac{d^s-1}{d-1}$ columns, $\norm{AB^T - AS^TSB^T}_F \leq \sqrt{\frac{\eps}{k}}\norm{A}_F\norm{B}_F$
	\end{itemize}
	The first property can be shown by Lemma \ref{lem:kspace} since we pick $m_i = \Omega(ik^2)$ and the second property can be shown by Lemma \ref{lem:prob_hd} since we pick $m_i = \Omega(\frac{ik}{\eps})$.
	Also, Theorem 3.1 of \cite{clarkson2009numerical} implies Lemma 4.2 of \cite{clarkson2009numerical} which means there is a matrix $Z$ such that $\norm{UZ - A_s}_F \leq (1+\eps) \norm{A_s - [A_s]_k}_F$ in our context.
	Combining Lemma 4.3 of \cite{clarkson2009numerical}, we have 
	\[
	\norm{U[U^TA_s] - A_s}_F\leq (1+\eps) \norm{A_s - [A_s]_k}_F \numberthis\label{eqn:first}
	\]
	Now, Lemma \ref{lem:kspace} implies Lemma 2.1 from \cite{kannan2014principal} and further implies 
	\[
	\norm{WW^TU^TA_s - A_s}_f \leq (1+\eps) \norm{A_s - [A_s]_k}_F \numberthis\label{eqn:second}
	\]
	by setting $k'$ in Lemma \ref{lem:kspace} be $m$ and picking $r_i = \Theta(\frac{3^im^2}{\eps^2})$. 
	Using equation (\ref{eqn:first}) and (\ref{eqn:second}) in the proof of Theorem 1.1 from \cite{kannan2014principal}, we have our conclusion $\norm{A_s - UWW^TU^TA_s}_F^2 = \norm{A_s - VV^TA_s}_F^2 \leq (1+\eps)\norm{A_s-[A_s]_k}_F^2$.
	
\end{proof}

\begin{theorem}~\label{thm:main_pca_hd}
	%For any $\eps,R,\alpha>0$, let $G$ and $H$ be randomly chosen from $\GSHD_{m,s}$ and $\GSHD_{r,s}$ with 
	%$m=\Theta(3^s(k^2+\frac{k}{\eps}))$,
	%$r=\Theta(\frac{3^sm^2}{\eps^2})$, and 
	%$s=\Theta\left( \frac{\log \frac{4n^2\exp(2R^2)}{\alpha}}{\log \left(\frac{1}{2eR^2}\log \frac{4n^2\exp(2R^2)}{\alpha}\right)} \right)$
	Let $\eps,R,\alpha>0$ and $s=\Theta\left( \frac{\log \frac{4n^2\exp(2R^2)}{\alpha}}{\log \left(\frac{1}{2eR^2}\log \frac{4n^2\exp(2R^2)}{\alpha}\right)} \right)$.
	For $\Lambda^d_R = \{x \in \R^d \mid \|x\|_2 \leq R\}$ and $X \subset \Lambda_R^d$, and let $A_X$ be a pd matrix with elements $(A_X)_{i,j} = K(x_i,x_j) = \exp(-\norm{x_i-x_j}^2)$ for $x_i,x_j \in X$ and factorization $A_X = B_X B_X^T$.  
	%Compute the $n \times k$ matrix $V = UW$ as the product of orthogonal basis $U,W$ from $n \times m$ matrix $M_X$ with rows $G(x_i)$ and $n \times r$ matrix $N_X$ with rows $H(x_i)$ using $x_i \in X$ as described above. 
	Then with constant probability
	\[
	\norm{B_X - VV^T B_X}_F^2 \leq (1 + \eps) \norm{B_X - [B_X]_k}_F^2 + \alpha.
	\]
	The runtime to compute $V$ is $O(nds^2 + n\frac{3^{4s}(k^2+\frac{k}{\eps})^3}{\eps^2})$.
\end{theorem}

\begin{proof}
	%	Let $v_x^{(i)}$ be a vector in $\mathbb{R}^s$ with $j$th coordinate to be $\exp(-x_i^2)\sqrt{\frac{2^{j-1}}{(j-1)!}} x^{j-1}_i$ for any $x\in \mathbb{R}^d$.  \jeff{???}	
	By Lemma~\ref{lem:pca_hd}, we have
	\[
	\norm{A_s - VV^TA_s}_F^2 \leq (1+\eps)\norm{A_s-[A_s]_k}_F^2.  
	\] 
	
	From Lemma \ref{lem:tail_hd}, $v^T(B_XB_X^T - A_sA_s^T)v \leq \left(\sum_{i=1}^n \abs{v_i}\right)^2 \exp(2R^2)\left(\frac{2eR^2}{s}\right)^s \leq \alpha/n$ with our setting of $s$ as long as $\left(\sum_{i=1}^n \abs{v_i}\right)^2 \leq 4n$.
	Indeed the columns of $V$ are orthonormal, so the norm of each row of $I-VV^T$ is at most $2$, and thus $\left(\sum_{i=1}^n \abs{v_i}\right)^2 \leq 4n$.  
	
	We now have
	\begin{align*}
	\norm{B_X-VV^TB_X}_F^2
	& =
	\Tr((I-VV^T)B_XB_X^T(I-VV^T)^T) \\
	& \leq
	\norm{A_s-VV^TA_s}_F^2 + \Tr((I-VV^T)(B_XB_X^T - A_sA_s^T)(I-VV^T)^T) \\
	& \leq
	\norm{A_s-VV^TA_s}_F^2 + n \cdot (\alpha/n)
	\end{align*}
	
	Also by Lemma \ref{lem:gaussian_inner_hd}, $B_XB_X^T - A_sA_s^T$ is still positive definite. 
	Therefore, 
	\begin{align*}
	&
	\norm{A_s-[A_s]_k}_F^2 \\
	& =
	\norm{A_s-UU^TA_s}_F^2 \quad\quad \text{where $U$ is the matrix of top-$k$ left singular vectors of $A_s$} \\
	& \leq
	\norm{A_s - U'U'^TA_s}_F^2 \quad\quad \text{where $U'$ is the matrix of top-$k$ left singular vectors of $B_X$}\\
	& =
	\norm{B_X-[B_X]_k}_F^2 - \Tr((I-U'U'^T)(B_XB_X^T - A_sA_s^T)(I-U'U'^T)) \\
	& \leq
	\norm{B_X-[B_X]_k}_F^2 \quad\quad \text{recall that $B_XB_X^T - A_sA_s^T$ is positive definite}
	\end{align*}
	
	We can plug in everything.
	\begin{align*}
	\norm{B_X-VV^TB_X}_F^2
	& \leq
	\norm{A_s-VV^TA_s}_F^2 + \alpha \\
	& \leq
	\norm{A_s-[A_s]_k}_F^2 + \alpha\\
	& \leq
	\norm{B_X-[B_X]_k}_F^2 + \alpha
	\end{align*}
	
	To see the running time, it takes $O(\sum_{i=1}^s i(d+m_i\log m_i))$ to compute $G(\cdot)$ and \\$O(\sum_{i=1}^s i(d+r_i\log r_i))$ time to compute $H(\cdot)$.  Using that $r_i > m_i^2 > k^4$ and $m_i > 1/\eps$ then it takes less time to compute $H(\cdot)$ than $G(\cdot)$, and this runtime is $O(d s^2 + s^2 r_s \log r_s) = O(d s^2 + s^2 r \log r)$ since the $r_i$ values are exponentially increasing in $i$, and so $r_s = O(r)$ for $r = \sum_{i=1}^s r_i$.  
	The time to compute $M$ and $N$ is $n$ time longer.  
	
	We can compute the basis $U$ of $M$ in $O(nm^2)$ time, and the projection $U^T N$ in $O(nrm)$ time -- this step is the post-sketch bottlneck.  The basis $W$ takes $O(rm^2)$ time, and the final low rank basis $V = UW$ takes $O(nmk)$ time.  
	Thus the total runtime is 
	$O(n (ds^2 + s^2 r \log r) + nm^2 + nrm + rm^2 + nmk) = O(n(ds^2 + rm))$ using that $r > m^2 > k^4$ that $m > s^2 \log r$, and assuming $n > r$.  
	Now using $m = O(s^2(k^2 + k/\eps))$ and $r = O(s m^2/\eps^2) = O(s^3 (k^4 + k^2/\eps^2)/\eps^2)$ and we have a total time of
	$O(nds^2 + n s^4(k^2+\frac{k}{\eps})^3 / \eps^2)$.	
\end{proof}

\end{document}